\def\year{2022}\relax
\documentclass[letterpaper]{article} 
\usepackage{aaai22}  
\usepackage{times}  
\usepackage{helvet}  
\usepackage{courier}  
\usepackage[hyphens]{url}  
\usepackage{graphicx} 
\usepackage{natbib}  
\usepackage{caption} 
\DeclareCaptionStyle{ruled}{labelfont=normalfont,labelsep=colon,strut=off} 
\usepackage{algorithm}

\usepackage{newfloat}
\usepackage{listings}
\floatstyle{ruled}
\newfloat{listing}{tb}{lst}{}
\floatname{listing}{Listing}
\usepackage{bibunits}

\title{An Experimental Design Approach for Regret Minimization in Logistic Bandits}
\author {
    Blake Mason, Kwang-Sung Jun, Lalit Jain
}
\affiliations{
    Rice University, University of Arizona, University of Washington
}

\usepackage{amsthm,amsmath,bbm,amsfonts,amssymb}
\usepackage{dsfont} 
\usepackage{mathtools}
\usepackage{commath}
\mathtoolsset{showonlyrefs=false}

\usepackage{xspace}
\usepackage[utf8]{inputenc}
\usepackage[usenames,dvipsnames]{xcolor} 

\usepackage{url}
\usepackage{cleveref} 

\usepackage[export]{adjustbox}
\usepackage{algorithm}

\usepackage{graphicx,tabularx}
\usepackage{multirow,hhline}
\usepackage{booktabs}

\usepackage{capt-of}

\allowdisplaybreaks 

\usepackage{pbox} 

\usepackage[export]{adjustbox}

\usepackage{pifont}

\def\ddefloop#1{\ifx\ddefloop#1\else\ddef{#1}\expandafter\ddefloop\fi}
\def\ddef#1{\expandafter\def\csname c#1\endcsname{\ensuremath{\mathcal{#1}}}}
\ddefloop ABCDEFGHIJKLMNOPQRSTUVWXYZ\ddefloop
\def\ddef#1{\expandafter\def\csname b#1\endcsname{\ensuremath{{\boldsymbol{#1}}}}}
\ddefloop ABCDEFGHIJKLMNOPQRSUVWXYZabcdeghijklmnopqrtsuvwxyz\ddefloop
\def\ddef#1{\expandafter\def\csname h#1\endcsname{\ensuremath{\hat{#1}}}}
\ddefloop ABCDEFGHIJKLMNOPQRSTUVWXYZabcdefghijklmnopqrsuvwxyz\ddefloop 
\def\ddef#1{\expandafter\def\csname hc#1\endcsname{\ensuremath{\widehat{\mathcal{#1}}}}}
\ddefloop ABCDEFGHIJKLMNOPQRSTUVWXYZ\ddefloop
\def\ddef#1{\expandafter\def\csname t#1\endcsname{\ensuremath{\widetilde{#1}}}}
\ddefloop ABCDEFGHIJKLMNOPQRSTUVWXYZabcdefghijklmnopqrtsuvwxyz\ddefloop
\def\ddef#1{\expandafter\def\csname r#1\endcsname{\ensuremath{\mathring{#1}}}}
\ddefloop ABCDEFGHIJKLMNOPQRSTUVWXYZabcdefghijklmnopqrtsuvwxyz\ddefloop
\def\ddef#1{\expandafter\def\csname tc#1\endcsname{\ensuremath{\widetilde{\mathcal{#1}}}}}
\ddefloop ABCDEFGHIJKLMNOPQRSTUVWXYZ\ddefloop

\def\hmu{{\ensuremath{\hat{\mu}}} }

\DeclareMathOperator{\EE}{\mathbb{E}}

\DeclareMathOperator{\PP}{\mathbb{P}}

\DeclareMathOperator{\one}{\mathds{1}\hspace{-.18em}}

\newcommand{\lsim}{\mathop{}\!\lesssim}
\newcommand{\gsim}{\mathop{}\!\gtrsim}

\def\RR{{\mathbb{R}}}

\newcommand{\sr}[2]{ {\stackrel{#1}{#2}} }

\newcommand{\fr}[2]{ { \frac{#1}{#2} }}

\def\sign{\ensuremath{\text{sign}}}

\def\larrow{\ensuremath{\leftarrow}\xspace} 
 
\def\T{\ensuremath{\top}}  
\def\sig{\ensuremath{\sigma}\xspace}

\def\eps{\ensuremath{\varepsilon}\xspace}
\def\diag{\ensuremath{\mbox{diag}}\xspace}

\def\om{{\ensuremath{\omega}\xspace} }
\def\dt{{\ensuremath{\delta}\xspace} }
\def\sm{{\ensuremath{\setminus}\xspace} }

\def\lcl{\lceil}  
\def\rcl{\rceil}  

\def\gam{{\ensuremath{\gamma}\xspace} }

\makeatletter
\newcommand{\vast}{\bBigg@{3}}
\newcommand{\Vast}{\bBigg@{4}}
\makeatother

\def\cX{\ensuremath{\mathcal{X}}\xspace}

\def\lam{\ensuremath{\lambda}}

\def\hth{{\hat{ \theta}}}

\def\lam{{\ensuremath{\lambda}\xspace} }

\def\cS{{\ensuremath{\mathcal{S}}}}

\def\th{{\ensuremath{\theta}}}

\def\sign{\ensuremath{\normalfont{\text{sign}}}}

\def\cH{{\ensuremath{\mathcal{H}}}}

\def\cd{\cdot}

\DeclareMathOperator{\Supp}{{\mathsf{Supp}}}

\def\hlam{{\hat\lambda}}

\def\eff{{\mathsf{eff}}}

\def\opt{{\mathsf{opt}}}
\def\pes{{\mathsf{pes}}}

\def\dmu{\dot\mu}
\def\mc#1{\mathcal{#1}} 
\def\teff{{\ensuremath{t_\mathsf{eff}}}}
\def\naive{{\mathsf{naive}}} 
\def\tw{\mathsf{w}}

\newcommand{\HOMER}{\texttt{HOMER}\xspace}
\newcommand{\WAR}{\texttt{WAR}\xspace}

\usepackage{comment}
\usepackage{algpseudocode}[1]
\usepackage{MnSymbol} 

\usepackage[shortlabels]{enumitem}

\newtheoremstyle{plain2}
  {\topsep}   
  {\topsep}   
  {}  
  {0pt}       
  {\bfseries} 
  {.}         
  {5pt plus 1pt minus 1pt} 
  {}          
  
\theoremstyle{plain}
\newtheorem{thm}{Theorem}
\newtheorem{lem}[thm]{Lemma}
\newtheorem{cor}[thm]{Corollary}

\theoremstyle{plain2}
\newtheorem{remark}{Remark}

\newtheorem{definition}{Definition}

\newcommand{\htheta}{\hat{\theta}}

\newcommand{\mR}{\mathbb{R}}
\newcommand{\mudot}{\dot{\mu}}
\renewcommand{\ln}{\log}
\newcommand{\toremove}[1]{#1}
\renewcommand{\toremove}[1]{}

\begin{document}

\maketitle


\begin{abstract}
In this work we consider the problem of regret minimization for logistic bandits. The main challenge of logistic bandits is reducing the dependence on a potentially large problem dependent constant $\kappa$ that can at worst scale exponentially with the norm of the unknown parameter $\theta_{\ast}$. Abeille et al. (2021) have applied self-concordance of the logistic function to remove this worst-case dependence providing regret guarantees like $O(d\log^2(\kappa)\sqrt{\dot\mu T}\log(|\mathcal{X}|))$ where $d$ is the dimensionality, $T$ is the time horizon, and $\dot\mu$ is the variance of the best-arm. This work improves upon this bound in the fixed arm setting by employing an experimental design procedure that achieves a minimax regret of $O(\sqrt{d \dot\mu T\log(|\mathcal{X}|)})$. 
  Our regret bound in fact takes a tighter instance (i.e., gap) dependent regret bound for the first time in logistic bandits. 
We also propose a new warmup sampling algorithm that can dramatically reduce the lower order term in the regret in general and prove that it can replace the lower order term dependency on $\kappa$ to $\log^2(\kappa)$ for some instances.
  Finally, we discuss the impact of the bias of the MLE on the logistic bandit problem, providing an example where $d^2$ lower order regret (cf., it is $d$ for linear bandits) may not be improved as long as the MLE is used and how bias-corrected estimators may be used to make it closer to $d$.
\end{abstract}

\toremove{
\textbf{Notes for us to remember:}
\begin{enumerate}
    \item 7 content pages, up to 2 pages for references
    \item references come after appendix
    \item read the formatting instructions pdf and call for paper for specific details about ordering, banned commands etc. 
    \item AAAI is anal about use of color, certain commands, removing spacing
    \item no page numbering allowed
\end{enumerate}
}

\section{Introduction}
Linear bandits, which have gained popularity since their success in online news recommendation~\cite{li10acontextual}, solve sequential decision problems under limited feedback when each action (or arm) to be taken has a known feature vector deemed to predict the reward.
Specifically, at each time step $t$, the learner chooses an arm $x_t\in\RR^d$ from an available pool of arms $\cX$, and then receives a reward $y_t = x_t^\T\th_* + \eta_t$, where $\theta_{\ast}$ is unknown and $\eta$ is usually assumed to be zero-mean subGaussian noise. 
The goal of the learner is to maximize the total cumulative rewards over the time horizon $T$ by judiciously balancing between efficiently learning $\th_*$ (exploration) and using the learned knowledge on $\th_*$ to accumulate large rewards (exploitation).
Since the pioneering studies by~\citet{abe99associative} and~\citet{auer02using}, there have been significant developments in both theory~\cite{dani08stochastic,ay11improved,foster20beyond} and applications~\citep{li10acontextual,SawantHVAE18,Teo2016airstream}.

Many real-world applications, however, have binary rewards and are not captured by the additive noise setting.
For example, the seminal work by~\citet{li10acontextual} for contextual bandits considers a binary reward of click/no-click, yet they apply bandit algorithms based on linear models -- this is comparable to applying linear regressions to a binary classification task.
For binary rewards, the logistic linear model is natural when rewards are assumed to follow $y_t \sim \mathsf{Bernoulli}(\mu(x_t^\T\th_*))$ where $\mu(z) = 1/(1+\exp(-z))$ is the logistic function.
While the link function $\mu$ can be changed, the logistic function is worth more attention for two reasons: (i) it is extensively used in practice, even in state-of-the-art deep architectures whose last layer is the negative log likelihood of the logistic linear model, and (ii) if we were to use a trained network to compute the input to the last layer and take it as the feature vector for bandit tasks, the features are likely optimized to excel with the logistic model.

The first work on logistic bandits due to \citet{filippi10parametric} showed a regret bound of $\tilde O(d\kappa\sqrt{T})$ where $\kappa = \max_{\|x\|_2\leq 1} \mudot(x^{\top}\theta_{\ast})^{-1}$ and $\tilde O$ ignores polylogarithmic factors for all variables except for $\kappa$. 
Since $\kappa$ can be exponential in $\|\theta_{\ast}\|_2$, the key challenge in developing bandit algorithms in the logistic setting both theoretically and practically is to overcome this worst-case scaling. 
In the last few years, there has been a flurry of activity on this problem that exploits the \emph{self-concordance} of the logistic loss with the seminal work of~\citet{faury2020improved}. 
Recently, \citet{abeille2021instance} proposed a UCB style algorithm called \texttt{OFULog}, establishing a regret bound with the leading term of $\tilde O(d\ln^2(\kappa)\sqrt{\mudot(x_{\ast}^{\top}\theta_{\ast})T}+\kappa d^2\wedge (d^2+|\mc{X}_{-}|))$ where $\mc{X}_{-}\subset \mc{X}$ is a set of \textit{detrimental arms}. 
In the finite armed contextual bandit setting, \citet{jun2020improved} propose an improved fixed design confidence interval and adapted a SupLinRel style algorithm \cite{auer02finite} called \texttt{SupLogistic} to establish a regret scaling like $\tilde{O}(\sqrt{dT})$.
\texttt{SupLogistic} achieves a better dependence on $d$ and $\kappa$.
However, it has a worse dependence with $\dmu(x_*^\T\th_*)$ due to the changing arm set setting and makes a strong assumptions of stochastic contexts and bounded minimum eigenvalues.
The regret bound of \texttt{OFULog} is free of these assumptions, but the leading term is suboptimal.
We discuss key related work throughout the paper and postpone detailed reviews to our supplementary.

Motivated by the gaps in the regret bounds, we make the following contributions.

\textbf{Improved logistic bandit algorithm} (Section~\ref{sec:alg}): 
    We take an experimental design approach to propose a new bandit algorithm called \HOMER (H-Optimal MEthod for Regret) that achieves the best of the two state-of-the-art algorithms above: $\tilde O(\sqrt{d\dmu(x_*^\T\th_*)T\log(|\mc{X}|)} + d^2\kappa)$ in the fixed-arm setting where the lower order term matches the state-of-the-art \texttt{OFULog} in the worst case.
    In fact, we prove an even tighter instance-dependent (i.e., gap-dependent) regret bound of $O(\fr{d\dmu(x_*^\T\th_*) \ln(|\mc{X}|T)}{\Delta} + d^2\kappa) $ for the first time to our knowledge where $\Delta $ is the reward gap between the best arm among the suboptimal arms and the overall best arm.  
    
\textbf{Novel warmup algorithm} (Section~\ref{sec:warmup}): 
    While \HOMER achieves the best worst-case regret guarantee, it must be equipped with a warmup sampling procedure.
    Using a naive sampling procedure, \HOMER will incur $d^2\kappa$ regret during the warmup. 
    This stems from having to use fixed design confidence bounds \cite[Theorem 1]{jun2020improved}-- without them there is no known ways to achieve the factor $\sqrt{d}$ in the leading term of the regret bound when there are finitely many arms -- that require the observed arms and their rewards to satisfy so-called ``warmup'' condition (see~\eqref{eq:warmup} below).
    In order to improve its practical performance, we propose a novel adaptive warmup algorithm called \WAR (Warmup by Accepts and Rejects), which performs the warmup with much fewer samples than $\kappa d^2$ in general.
    We prove its correctness guarantee, show that for 1d it never spends more samples than the naive warmup, and present an arm-set dependent optimality guarantee.
    
\textbf{Conjectures on the dimension dependence of the fixed design inequalities} (Section~\ref{sec:dimension}): 
    Omitting the $\kappa$ dependence and logarithmic factors in this paragraph, all existing regret bounds for logistic bandits have an $d^2$ dependence in the lower order term, which is in stark contrast to an $d$ dependence in linear bandits~\cite{auer02using}.
    The consequence is that logistic bandit algorithms suffer a linear regret until $T=d^2$ in the worst case.
    Does this mean logistic bandits are fundamentally more difficult than linear bandits in high dimensional settings?
    While we do not have a complete answer yet, we provide a sketch of an argument that, when the MLE is used, such a $d^2$ dependence might be unimprovable.
    The argument starts from the classical fact that the MLE of generalized linear models (GLMs) are biased, e.g., ~\cite{bartlett53approximate}.
    Based on this fact, we observe that in order to obtain tight fixed design confidence bounds we need to perform oversampling of arms as a function of $d$ until the squared bias gets smaller than the variance.
    Furthermore, based on the known fact that in 1d the KT estimator~\cite{krichevsky81theperformance} is much less biased than the MLE~\cite{cox89analysis}, which we verify numerically as well, we propose a new estimator and conjecture that it may lead to the lower order regret term of $ O(d^{4/3})$. 

\section{Regret Minimization for Logistic Bandits}\label{sec:alg}



Let us formally define the problem.
We assume access to a finite and fixed set of arms $\mc{X}\subset  \{x\in\mathbb{R}^d: \|x\|_2 \le 1 \}$. At each time $t\geq 1$ the learner chooses an arm $x_t\in \mc{X}$ and observes a Bernoulli reward $y_t\in \{0,1\}$ with 
\[\mathbb{E}[y_t|x_t] = \mu(x_t^{\top}\theta_*)~.\]
Let $x_* = \arg\max_{x\in \mc{X}}x^\T \theta_{\ast}$.
For ease of notation, we assume that $x_{\ast}$ is unique though this condition can be relaxed.
The goal of the learner is to minimize the \emph{cumulative (pseudo-)regret} up to time $T$:
$R_T := \sum_{t=1}^T \mu(x_{\ast}^{\top}\theta_{\ast}) - \mu(x_t^{\top}\theta_{\ast})$.

\textbf{Notations.} Let $ {\mc{F}_t}$ be the sigma algebra generated by the set of rewards and actions of the learner up to time $t$, i.e., $\sigma(x_1, y_1, \ldots, x_{t-1}, y_{t-1})$. 
We assume that the learner has knowledge of an upper bound $S$ on $\|\theta_{\ast}\|$. 
Define $ {\kappa} := \max_{\|x\|\leq 1} \mudot(x^\T \theta_\ast)^{-1}$ and $ {\kappa_0} := \max_{x\in \cX} \mudot(x^\T \theta_\ast)^{-1}$, the inverse of the smallest derivative of the link function among elements of $\mc{X}$. 
Denote by $ {\triangle_\cA}$ the set of probability distributions over the set $\cA$.
Let $\Supp(\lambda)$ with $\lambda \in \triangle_\cA$ be the subset of $\cA$ for which $\lambda$ assigns a nonzero probability.
We use $A\lsim B$ to denote $A$ is bounded by $B$ up to absolute constant factors.



\textbf{Logistic Regression.}
We review logistic regression. 
Assume that we have chosen measurements $x_1, \ldots, x_t$ to obtain rewards $y_1,\ldots, y_t$. 
The \emph{maximum likelihood estimate} (MLE), $\hat{\theta}$ of $\theta_{\ast}$ is given by, 
\begin{align}\label{eq:mle}
    \hat{\theta} = \arg\max_{\theta\in \mR^d} \sum_{s=1}^t y_s\log(\mu(x_s^{\top}\theta))+(1-y_s)\log(1-\mu(x_s^{\top}\theta))
 \end{align}
The Fisher information of the MLE estimator is  $H_t(\theta_{\ast}) := \sum_{s=1}^t \mudot(x_s^{\top}\theta_{\ast}) x_sx_s^{\top}$.
Obtaining (near-)optimal regret hinges on the availability of tight confidence bounds on the \emph{means} of each arm. 
For technical reasons, tight confidence bounds (i.e., without extra factors such as $\sqrt{d}$ in the confidence bound like~\cite{faury2020improved}) require the data to observe the \textit{fixed design} setting: for each $s\in [t]$, $y_s$ is conditionally independent of $\{x_s\}_{s=1}^t\setminus \{x_s\}$ given $x_s$.
Recent work by \citet{jun2020improved} provide a tight finite-time fixed design confidence interval on the natural parameter of $x^\T \theta_\ast$ for any $x\in \mc{X}$. 
For regret minimization,  we instead require estimates of the mean parameter $\mu(x^\T \theta_\ast)$  each arm $x\in\cX$. 
Define ${\gamma(d)} \!:=\max\{\! d + \log(6(2\!+\!\teff)/\dt), 6.1^2\log(6(2\!+\!\teff)/\dt)\}$ where $t_{\eff}$ is the number of distinct vectors in $\{x_s\}_{s=1}^t$. We refer to the following assumption on our samples as the \emph{warmup condition}:
 \begin{equation}\label{eq:warmup}
      {\xi^2_t} := \max_{1\leq s\leq t} \|x_s\|^2_{H_{t}(\th_*)^{-1}} \le \fr{1}{\gamma(d)}~.
 \end{equation}
\begin{lem}\label{lem:mean_param_bound}
Fix $\delta \leq e^{-1}$. Let $\hat{\theta}_t$ denote the MLE estimate for a fixed design $\{x_1, \ldots, x_t\} \subset \cX$. Under the warmup condition~\eqref{eq:warmup}, with probability $1-\delta$, we have, $\forall x \in \cX$,
\begin{align} \label{eq:lem-mean_param_bound}
    &|x^\T(\hth_t - \th_*)| \le 1 \text{ and } 
    |\mu(x^\T\hat{\theta}_t) - \mu(x^\T{\theta}_\ast)| \\
    &\hspace{0.25cm}\leq 4.8\dot{\mu}(x^\T\theta_\ast)\|x\|_{H_t(\theta_\ast)^{-1}}\sqrt{\log(2(2 + t_\text{eff})|\cX|/\delta)}\nonumber. 
\end{align}
\end{lem}
The proof relies on Theorem 1 of \cite{jun2020improved} and the (generalized) self-concordance of the logistic loss (\citet[Lemma 9]{faury2020improved}). 
While similar results were used in the proof of SupLogistic in \citet{jun2020improved}, for our bandit algorithm below it is crucial to guarantee $\forall x\in\cX, |x^\T(\hth_t - \th_*)|\le 1$ as well.
As far as we know, this is the first non-trivial confidence bound on the mean parameter for logistic models. To contextualize this result, consider that via the delta-method, 
$    \sqrt{t}(\mu(x^\T\htheta) - \mu(x^\T\theta_\ast))   \xrightarrow{D} \mc{N}\left(0, \dmu(x^\T\theta_\ast)^2\|x\|_{H_t(\theta_\ast)^{-1}}^2\right)$.
Hence, Lemma~\ref{lem:mean_param_bound} guarantees an asymptotically tight normal-type tail bound up to constant factors provided the warmup condition holds. 



\textbf{Experimental design.}
We leverage our tight confidence widths on the means $\mu(x^\T \theta_\ast)$ given in Lemma~\ref{lem:mean_param_bound} to develop a novel logistic bandit algorithm.
Motivated by the form of the confidence width on the mean parameter, we consider the following experimental design problem:
\begin{align*}
   {h^*} &:= \min_{\lambda\in \triangle_{\mc{X}}}\max_{x\in \mc{X}} \mudot(x^\T \theta_*)^2\|x\|_{H_{\lambda}(\theta_*)^{-1}}^2 \\
   &\text{~~ where ~~} H_{\lambda}(\theta) = \sum_{x\in \mc{X}} \lambda_x \mudot(x^\T \theta) xx^\T~,
\end{align*}
which we refer to as an \emph{$H$-optimal design}. 
That is, we want to find an allocation of samples over $\mc{X}$ minimizing the worst case confidence-width of Lemma~\ref{lem:mean_param_bound} for any $x\in \mc{X}$.
Note that $h^*$ depends on the arm set $\cX$, though we omit it for brevity.
This experimental design is closely linked to the $G$-optimal design objective for the MLE in an exponential family. Indeed, in our setting, the $G$-optimal design is 
\begin{align}\label{eq:g-optimal}
   {g^*} := \min_{\lambda\in\triangle_\cX} \max_{x\in \mc{X}} \|x\|^2_{H_\lambda(\theta_*)^{-1}}
\end{align}
We point out that in the setting of linear bandits (i.e. the linear GLM), the celebrated Kiefer-Wolfowitz theorem states that the optimal value of this objective is just $d$  for any choice of $\mc{X}$ \cite{kiefer60theequivalence}. 
Hence, for any $\theta\in \mathbb{R}^d$, letting $\mc{Y} = \{\sqrt{\mudot(x^\T \theta)}x, x\in \mc{X}\}$ we see that
\begin{align*}
  h^*& = \min_{\lambda\in \triangle_{\mc{X}}}\max_{x\in \mc{X}} \mudot(x^\T \theta)^2\|x\|_{H(\theta_*)^{-1}}^2 \\
  &
  \leq \frac{1}{4} \min_{\lambda\in \triangle_{\mc{Y}}}\max_{y\in \mc{Y}} \|y\|^2_{(\sum_{y\in \mc{Y}} \lambda_y yy^{\top})^{-1}} 
  \leq \frac{d}{4}
\end{align*}
where we use $\mudot(z) \leq 1/4$ for the first inequality and the Kiefer-Wolfowitz theorem for the last inequality.
Since $\mudot(x^\T \theta)$ decays exponentially fast in $|x^\T \theta|$, this bound is overly pessimistic and in practice the values of the $H$-optimal will be much smaller.
In contrast, for the logistic setting, the $G$-optimal design objective may be large and we only have a naive bound $f_G(\mc{X})\leq \kappa_0 d$ obtained by naively lower bounding $H(\lambda) \geq \kappa_0\sum_{x\in \mc{X}} \lambda_x xx^\T $. In general these two criteria can produce extremely different designs. We provide an example where these designs are very different in our supplementary, see Figure~\ref{fig:exp_designs}.

Though experimental design for logistic models is an important and abundant topic, e.g., social science  applications where tight estimates on entries of $\theta$ are required for causal interpretation \cite{erlander05welfare}, as far as we know, the design above has not previously been proposed.
The closest that we are aware of is \citet{russell18design} that considers $\dmu(x^\T\th) \|x\|^2_{H(\th)^{-1}}$ rather than $\dmu(x^\T\th)^2 \|x\|^2_{H(\th)^{-1}}$. 
Most existing studies on optimal design in nonlinear models study theoretical properties of the design problem assuming the knowledge of $\th^*$ and then uses a plugin estimate.
However, they hardly study under what conditions the plugin estimate must satisfy for the plug-in design problem to closely approximate the true design problem and how one can efficiently collect data for the plugin estimate.
In contrast, we address these in our paper for the $H$- and $G$-optimal design problem for logistic models.

\subsection{From experimental design to regret minimization.}
We now introduce our primary algorithm, \HOMER (\textbf{H}-\textbf{O}ptimal \textbf{ME}thod for \textbf{R}egret) which is centered around the $H$-optimal design objective and the confidence bound on the mean parameter given in Lemma~\ref{lem:mean_param_bound}. To build an initial estimate $\htheta_0$ of $\theta_\ast$, the algorithm begins by calling a warmup procedure \texttt{WarmUp} with the following guarantee: 

\begin{definition}
  A warmup algorithm $\cA$ is said to be $\dt$-valid if it returns an estimator $\hth_0$ such that it is certified to have $\PP(\forall x\in \cX: |x^\T(\hth_0 - \th_*)| \le 1 ) \ge 1-\dt$.
\end{definition}
One natural attempt would the aforementioned experimental design approach of solving $ g^{\ast} = \min_{\lam \in \triangle_\cX } \max_{x\in\cX} \|x\|^2_{H_\lam (\th_*)^{-1}}  $.
We are guaranteed a solution $\lam^*$ with a support of at most $d(d+1)/2$ via Caratheodory's theorem; e.g., see~\citet[Theorem 21.1]{lattimore2020bandit}.
We can then pull arm $x$ exactly $\lceil \lam^*_x g^* \gamma(d)\rceil$ times to satisfy~\eqref{eq:warmup}, which in turns makes the MLE $\hth_0$ trained on these samples to perform a $\dt$-valid warmup.
However, we do not know $\th_*$.
Fortunately, when an upper bound $S$ on $\th_*$ is known, we can we consider the following naive warm-up procedure:
\begin{align}\label{eq:warmup-naive}
  {g^{\naive}} =&  \min_{\lam\in\triangle_{\cX}} \max_{x\in\cX} \|x\|^2_{(H^\naive_\lam)^{-1}}  \nonumber\\
  & 
  \text{~~ where ~~}  {H^\naive_\lam} = \sum_x \lam_x \dmu(\|x\|S)x x^\T~.
\end{align}
%
Let $\hat{\lambda}^{\naive}$ be the solution to this problem. Since $x^\T \th_* \le \|x\|\|\th_*\|$, we have $g^\naive \ge g^*$, so we can guarantee that the warmup condition~\eqref{eq:warmup} is satisfied when pulling each arm $x$ exactly  $\lceil \hat{\lambda}^{\naive}_x g^\naive \gamma(d)\rceil$ times.
Computing the MLE $\hth_0$ with these samples leads to a $\dt$-valid warmup with the sample complexity of $O(g^\naive\gam(d) + d^2)$ which in the worse case is  $O(\kappa d^2)$.
We discuss a more efficient warmup algorithm in Section~\ref{sec:warmup}; in this section, let us use this naive warmup procedure.


The pesudocode of \HOMER can be found in Algorithm~\ref{alg:mix-exp-design-regret}.
In each round $k$, \HOMER maintains an active set of arms $\cX_k$ and computes two experimental design objectives over $\cX_k$, each using the MLE estimate from the previous round $\htheta_{k-1}$. The main experimental design, denoted $\lambda_k^H$,
is the H-optimal design in Line~\ref{line:H-opt} which ensures the gap $\mu(x_\ast^\T\theta_\ast) - \mu(x^\T \theta_\ast)$ is estimated to precision $2^{-k}$. This allows us to remove any arm $x$ whose gap is significantly larger than $2^{-k}$. 
The second optimal design, denoted $\lambda_k^G$,
is a G-Optimal design given in Line~\ref{line:G-opt}. It is necessary to ensure that the warmup condition holds in each round as samples are not shared across rounds. In order to trade off these two, possibly competing design criteria, \HOMER computes a mixture of $\lambda_k^H$ and $\lambda_k^G$ (denoted $\lambda_k$) with approximately $1-2^{-k}$ of the mass being given to the $H$-optimal design $\lambda_k^H$. 
Rather than sampling directly from this distribution
\HOMER relies on an efficient rounding procedure, \texttt{Round}. 
Given a distribution $\lambda$, tolerance $\epsilon$, and number of samples $n$, \texttt{Round} returns an allocation $\{x_i\}_{i=1}^n$ such that $H_n(\theta)$ is within a factor of $(1+\epsilon)$ $H(\lambda, \theta)$ for any $\theta \in \mathbb{R}^d$ provided $n \geq r(\epsilon)$ for a minimum number of samples $r(\epsilon)$. Efficient rounding procedures are discussed in \citet{fiez19sequential}. Recent work \citep[]{camilleri2021high} has shown how to avoid rounding for linear bandits through employing robust mean estimators, but it remains an open question for logistic bandits. 
\HOMER passes the mixed distribution to \texttt{Round} and samples according to the returned allocation to compute an MLE estimate $\htheta_k$. Finally it removes suboptimal arms using plug in estimates of their means $\mu(x^\T\htheta_k)$. 
\begin{algorithm}[t] \caption{\HOMER: H Optimal MEthod for Regret}
  \label{alg:mix-exp-design-regret}
  \small
  \begin{algorithmic}
    \Require{$\epsilon$, $\delta$, $\cX$, $\kappa_0$ }
    \State{$k=1, \cX_1 = \cX$, $\gamma(d, n, \delta) := \max\{d + \log(6(2+n)/\delta), 6.1^2\log(6(2+n)/\delta)\} $}
    \State {$\hat\theta_0\leftarrow$ \texttt{WarmUp}($\cX$)}  
    \While{$|\cX_k| > 1$ } 
    \State  $\delta_k = \delta/(4(2+|\cX|)|\cX|k^2)$
    \State{$\lambda_{k}^H = \arg\min_{\lambda\in \triangle_{\cX_k}} {\hat h_k(\lam)}$ \label{line:H-opt}}
    \State{\hspace{0.5cm}for ${\hat h_k(\lam)} := \max_{x\in\cX_k} \dot{\mu}(x^\T\hat{\theta}_{k-1})^2\|x\|_{H_\lambda(\hat{\theta}_{k-1})^{-1}}^2$}
    \State{$\lambda_{k}^G = \arg\min_{\lambda\in \triangle_{\cX_k}} \del{ {\hat g_k(\lambda)}:=
        \max_{x\in\cX_k} \|x\|_{H_\lambda(\hat{\theta}_{k-1})^{-1}}^2} $\label{line:G-opt}} 
    \State{$n_k^H = \lceil 6(1+\epsilon)6.1^2 3^3 2^{2k} \hat h_k(\lam_k^H)\log(\delta_k^{-1}) \rceil$}
    \State{$n_k^G = \lceil 6(1+\epsilon)\gamma(d, |\cX_k|, \delta_k) \hat g_k(\lam_k^G) \rceil$}
    \State{$\Tilde{\lambda}_{k,i} = \max\left\{\frac{n_k^H}{n_k^H + n_k^G}\lambda_{k,i}^H, \frac{n_k^G}{n_k^G + n_k^H}\lambda_{k,i}^G \right\}, \forall i\in[n], $}
    \State{\hspace{0.5cm}where $\lambda_{k,i}^H$ and $\lambda_{k,i}^G$ are the $i$-th entry of $\lambda_k^H$ and $\lambda_k^G$}
    \State{$\lambda_{k,i} = \Tilde{\lambda}_{k,i} /\sum_{j=1}^n\Tilde{\lambda}_{k,j}$ and $n_k = \max\{n_k^H + n_k^G, r(\epsilon)\}$}
    \State{$x_1, \ldots, x_{n_k} \larrow \texttt{Round}( n_k, \lambda_k, \epsilon) $} 
    \State{Observe $y_1, \cdots, y_{n_k}$, compute MLE $\hat\theta_k$ with $ \{(x_i,y_i)\}_{i=1}^{n_k}$. }
    \State{$\cX_{k+1}\leftarrow \cX_k \setminus \left\{x\in \cX_k:\underset{x'\in \cX_k}\max\mu(x'^\T\hat{\theta}_k) - \mu(x^\T\hat{\theta}_k) \geq 2\!\cdot \!2^{-k}\right\}$}
    \State{$k\gets k+1$}
    \EndWhile
    \State Continue to play the unique arm in $\mc{X}_k$ for all time.
  \end{algorithmic}
\end{algorithm}

\textbf{Theoretical Guarantees.}
We now present theoretical guarantees of \HOMER.

\begin{thm}\label{thm:mix_regret_complexity}

  Fix $\delta \leq e^{-1}$ and suppose \texttt{WarmUp} draws at most $T_B$ samples and incurs regret at most $R_B$. 
  Define $ {T'} := T - T_B$ and assume $T' > 0$.
  Choose a rounding procedure with $r(\eps) = O(d/\eps^2)$  (e.g., \citet[Appendix B]{fiez19sequential}) and set $\eps = O(1)$.
  Let $\Delta := \min_{x\in \mc{X} \sm \{x_*\}} \mu(x_{\ast}^{\top} \theta_{\ast}) - \mu(x^\T  \theta_*)$
  Then, with probability at least $1-2\delta$, 
  \HOMER obtains a regret within a doubly logarithmic factor of 
  \begin{align*}
    R_{B} &+ \min_{\nu\geq 0} \del{ T \nu + \frac{d \dmu(x_*^\T\th_*)}{\Delta \vee \nu} \ln\del{\fr{|\cX|}{\dt}} + d\ln\del{\frac{1}{\Delta \vee \nu}} } \\
    &+ d\kappa_0\del{d + \ln\del{\frac{|\cX|}{\dt}}}~.
  \end{align*}
\end{thm}

\begin{remark}
Using the naive warmup~\eqref{eq:warmup-naive}, the one can show that $R_B = O\left(d^2\kappa\log(|\cX|/\delta)\right)$.
\end{remark}
The \texttt{OFULog} algorithm of
\citep{abeille2021instance} achieves a regret bound of $\tilde O(d\ln^2(\kappa)\sqrt{\mudot(x_{\ast}^{\top}\theta_{\ast})T} + d^2\kappa)$ where the lower order term may be improved for the case where there are few sub-optimal arms so the lower order term scales with the number of arms without the factor $\kappa$. 
\texttt{SupLogistic} of \citet{jun2020improved} follows SupLinRel-type sampling scheme to eliminates arms based on their natural parameter estimates and achieves a regret of $\tilde{O}(\sqrt{dT} + \kappa^2 d^3)$.\footnote{
    \citet{jun2020improved} in fact have reported that the lower order term is $O(\kappa^2 d)$, but there is a hidden dependency on $d$.
    This is because they assume that the expected minimum eigenvalue is at least $\sig_0^2$, but $\sig_0^2$ is $1/d$ at best. Our reported rate is for the best case of $\sig_0^2 = 1/d$.
}
To compare the leading term of the regret bound, SupLogistic achieves a better dependence on $d$ and $\kappa$.
However, it has a worse dependence with $\dmu(x_*^\T\th_*)$ due to the changing arm set setting and makes strong assumptions of stochastic contexts and bounded minimum eigenvalues.
The regret bound of OFULog is free of these assumptions and have a better lower-order terms, but the leading term is suboptimal.

In the following Corollary, we further upper bound the result of Theorem~\ref{thm:mix_regret_complexity} in order to compare to the results of \citep[]{abeille2021instance} and \citep{jun2020improved} and show that \HOMER enjoys state of the art dependence on $d$, $\dmu(x_\ast^\T\theta_\ast)$, and $\log(\kappa)$ simultaneously and also matches the state-of-the-art worst-case lower order terms $d^2\kappa$. 
\begin{cor}\label{cor:minimax_regret}
  Suppose we run \HOMER with $\dt = 1/T$ with the naive warmup~\eqref{eq:warmup-naive}  in the same setting as Theorem~\ref{thm:mix_regret_complexity}.
  Then, \HOMER satisfies 
  \begin{align*}
    \EE[R_T] = &\hat O\left(\left(  \sqrt{dT\dmu(x_\ast^\T\theta_\ast)\log(|\cX|T)} \right.\right. \\
    &\hspace{0.75cm} \left.\left.\wedge \frac{d \dmu(x_*^\T\theta_*) \log(|\cX| T)}{\Delta}  \right) + d^2\kappa \log(|\cX|T) \right)
  \end{align*}
  where $\hat O$ hides doubly logarithmic factors.
\end{cor}

This highlights that \HOMER simultaneously enjoys the optimal dependence on $d$ exhibited by \texttt{SupLogistic} and the state of the art dependence on $\dmu(x_\ast^\T\theta_\ast)$ seen in \texttt{OFULog}. Furthermore, \HOMER avoids a dependence on $\log(\kappa)$ in its leading term.


\section{Warmup by Accepts and Rejects (\WAR)}\label{sec:warmup}

We now describe our novel and efficient warmup procedure that is $\dt$-valid. 
The problem with the naive warmup~\eqref{eq:warmup-naive} is that when arms in the support of the design have norm close to 1, then $g^\naive$ will scale with $\kappa d$ leading a to $\kappa d^2$ regret lower order term.
In this section, we propose a novel warmup algorithm called Warmup by Accepts and Rejects (\WAR) that can significantly reduce the number of samples while provably being never worse than the naive warmup.

\textbf{Inuition from 1d.}
Before describing the method, we provide some intuition from the case of $d=1$ with $\cX = [-1, 1]$.
In this case, the design problem is simplified because of the fact that $\lambda^{\ast}$, the solution to the $G$-optimal problem~\ref{eq:g-optimal}, is supported on only one arm.
Thus, it suffices to find $ {x^{\dagger}}$: 
\begin{align}
   {x^{\dagger}} &= \arg \max_{x\in\cX} \dot \mu(x\th_*) x^2 \label{eq:Hoptimal-1d}.
\end{align}
Without loss of generality, assume that $|\th_*| \ge \arg\max_{z\in \RR} \dot\mu(z) z^2 = 2.399..$ ( otherwise $\kappa_0 \le 13.103$ is not too large and we can employ the naive warmup). 
Then, 
\begin{align}
  \max_{x\in[-1,1]} \dot\mu(x\th_*) x^2 &= \fr{1}{(\th_*)^2} \max_{x\in[-1,1]} \dot\mu(x\th_*) x^2 (\th_*)^2 \\
  &= \fr{1}{(\th_*)^2}   \max_{z\in [-|\th_*|, |\th_*|]} \dot\mu(z) z^2 ~\sr{(a)}{=}~ \fr{0.439..}{(\th_*)^2} \nonumber \label{eq:Hoptimal-1d-cont-soln}
\end{align}
where $(a)$ is by the assumption $|\th_*| \ge 2.399..$ and numerical evaluation, and $x^{\dagger} = \arg\max_{x\in[-1,1]} \dot\mu(x\th_*) x^2 = \pm \fr{2.399..}{|\th_*| }$.

\begin{figure}[h]
\centering
        \includegraphics[width=0.8\linewidth]{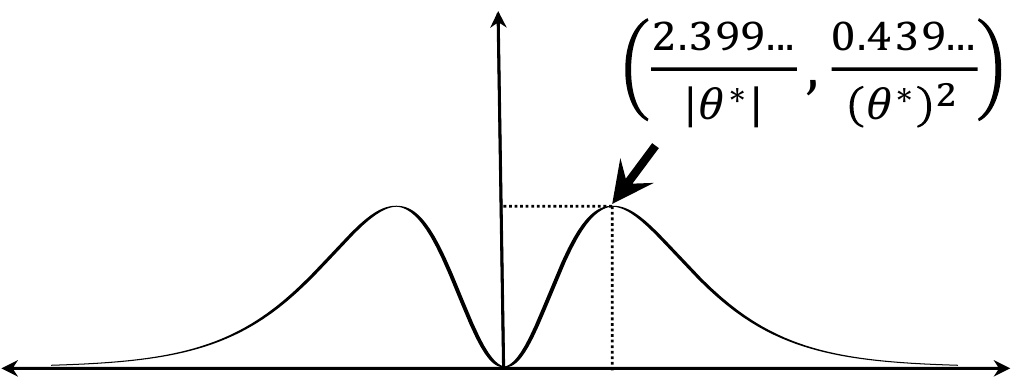}
    \caption{The objective function of~\eqref{eq:Hoptimal-1d} with $\cX =  [-1,1]$.}
    \label{fig:hopt-1d}
\end{figure}
We summarize the solution of the optimization problem above in Figure~\ref{fig:hopt-1d}.
We make two observations from this 1d example. Firstly, somewhat surprisingly the best design for the warmup does not always choose the arm with the largest magnitude unlike $G$-optimal design in the linear case. Secondly, in the best case, the number of samples needed to ensure the warmup is $O(\theta_{\ast}^2)\gamma(1)$.
Thus, we speculate that, for $d>1$, only $O(\|\th_*\|^2)d\gamma(d)$ samples may be needed for the warmup, which is significantly smaller than $\kappa d\gamma(d)$. 

The challenge is that we do not know $\theta^{\ast}$. 
However, we can use a halving procedure to find a constant factor approximation of $|\theta^{\ast}|$ in  $O(\ln(|\th_*|)$ sample complexity.
The key idea is that by choosing an arm $x$ s.t. $|x| \approx 1/|\th_*|$, the rewards conditioned on $x$ must have high variance, guaranteeing $|x\th_*|$ is sufficiently small.
Thus, starting from the arm $x=1$, we can sample until verifying that the reward variance conditional on $x$ is either small enough (e.g., $|x\th_*| \ge 1$) or large enough (e.g., $|x\th_*| \le 2$), which can be done using confidence bounds such as empirical Bernstein's inequality.
Once we verify that the variance is small enough, it means that $|x|$ is large enough, so we can then move on to the next arm $x=1/2$ (i.e., halving).
We repeat this process until we identify an arm whose variance is large enough, which means that we have approximately solved~\eqref{eq:Hoptimal-1d}.
It is easy to see that this procedure terminates in $\sim \ln|\th_*|$ iterations, spending total $\sim \ln|\th_*|$ samples.

Note that finding the arm $x\approx1/|\th_*|$ alone is not sufficient because we need to certify that the warmup condition~\eqref{eq:warmup} holds.
For this, we realize that the series of accept/rejects form a confidence bound on $\th_*$.
Specifically if $\hat{x}$ is the arm that is accepted in the last iteration we have that $\theta_{\ast}\in C:= \{\theta\in \mathbb{R}: 1/(2|\hat x|) \le |\th_*| \le 2/|\hat x|\}$ using the fact that $2\hat{x}$ was rejected. 
We can then solve the design $\min_{\lambda}\max_{x\in [-1,1], \theta\in C} \|x\|_{H_{\lambda}(\theta)^{-1}}^2$ and then sample to certify~\eqref{eq:warmup} with high probability. 


\textbf{Warmup by Accepts and Rejects (\WAR).} 
The remaining challenge is to extend the previous halving procedure to generic discrete arm sets and also to the multidimensional case.
This leads to our algorithm \WAR described in Algorithm~\ref{alg:war}.
Before describing the details, we introduce key quantities.
Assuming an arm $x$ is pulled $N$ times, let $\hmu_x$ be the empirical mean of the observed rewards.
We now construct valid lower and upper confidence bounds $ {L_x}$ and $ {U_x}$ on $|x^\T\th_*|$ where $L_x$ can be 0 and $U_x$ can be $\infty$.
To do so, for each arm $x$ we will use an anytime version of the empirical Bernstein inequality~\cite{mnih08empirical} which has the following confidence width:
\begin{align*}
     {W_x}  := \sqrt{\fr{\hmu_x(1-\hmu_x) 2\ln(3/\dt_N)}{N}} + \fr{3\ln(3/\dt_N)}{N} 
\end{align*}
where $ {\dt_N} = |\cX|N(N+1)/\dt$. 
\begin{thm}[\citet{mnih08empirical}]
  Define $  {\cE} = \cbr{\forall x \in \cX, \forall N\ge 1, \mu(x^\T\th_*) \in [\hmu_x - W_x, \hmu_x + W_x } $.
  Then, $\PP(\cE) \ge 1-\dt$.
\end{thm}
We can then define $ {L_x} = \mu^{-1}(0 \vee (\hmu_x - W_x))$ and $ {U_x} = \mu^{-1}(1 \wedge (\hmu_x + W_x)) $. 

The pseudocode of \WAR can be found in Algorithm~\ref{alg:war}. Let $ {m}$ be the stage index and define $ {\cH_m}$ be the set of arms that were pulled up to (and including) the stage $m$.
Let $ {S}$ be a known upper bound on $\|\th_*\|$ and let $ {\cB_d(S)}$ be the L2-ball of radius $S$.
We define a confidence set on $\th_*$:
\begin{align*}
   {C_m} = \cbr{\th\in \cB_d(S): |x^\T\th| \in [L_x, U_x], \forall x \in \cH_{m}}
\end{align*}
Let $ {\dmu^\opt_m(x)}$ be the optimistic estimate and $ {\dmu^\pes_m(x)}$ be the pessimistic estimate defined by
\begin{align*}
  \dmu^\opt_m(x) = \max_{\th \in C_m} \dmu(x^\T\th) ~~~~ \text{ and } ~~~~
  \dmu^\pes_m(x) = \min_{\th \in C_m} \dmu(x^\T\th)
\end{align*}
Define the accept and reject event on line 5 as
\begin{align*}
   {\mathsf{Accept}(x)} := \{ U_x <  {U}\} \text{~~ and ~~}  {\mathsf{Reject}(x)} := \{ L_x >  {L}\}
\end{align*}
for some $0<L<U$.
\WAR consists of two parts.
The first part, which we call optimistic probing, is the halving idea above extended to higher dimensions. 
The difference from 1d is that we first find $d$ arms that form a good basis, under the assumption that their variances are not small.
We then perform accepts/rejects to filter out arms with small variances that would likely introduce $\kappa_0$ dependency in the planning.
This filtering is done using $\dmu^\opt(x)$ because we do not want arms whose variances are small even in the best case.
Note we use the threshold $L/r$ rather than $L$ in line~\ref{line:elimination} in order to get the halving effect.
The second part called Pessimistic Planning simply computes the design based on the pessimistic estimate of the variances $\dmu^\pes(x)$, which allows the resulting sample assignments to certify the warmup condition~\ref{eq:warmup}.
We provide a practical and detailed version of \WAR in our supplementary.
\newcommand{\war}{\text{\texttt{WAR}}}
\begin{algorithm}[t]
  \caption{Warmup by Accepts and Rejects (\WAR)}
  \label{alg:war}
  \small
  \begin{algorithmic}[1]
    \Require{ Arm set $\cX$, parameters $L < R$, $r > 1$.}
    \State Set $ {\cS} = \cX$.
    \State \textbf{[Optimistic probing]}
    \For {$m=1,2,\ldots$} 
        \State Solve $\min_{\lam\in\triangle_\cS} \max_{x\in \cS} \|x\|^2_{V_\lam^{-1}}$ where $V_\lam = \sum_{x\in\cS} \lam_x x x^\T$ to obtain a 2-approximate solution supported on $O(d\ln(\ln(d)))$ points; see Remark~\ref{rem:gopt} below. \label{line:gopt}
        Call this solution $ {\hlam^{(m)}} $. \label{line:2approx}
        \State For every arm $x \in \Supp(\hlam^{(m)} )$, pull it until we either accept or reject (if it was pulled previously, skip sampling and reuse the accept/reject result).
        \If {all the arms in $\Supp(\hlam^{(m)})$ are accepted}
            \State \textbf{break}
        \EndIf
        \State $\cS \larrow \cS \sm \cbr{x \in \cS: \dmu^\opt_m(x) \le \dmu\del{\fr{L}{r} }}$. \label{line:elimination}
    \EndFor
    \State \textbf{[Pessimistic planning]}
    
    \State{Let  $\hlam^{\war}$ and $g^{\war}$ be the solution and the objective of $\arg \min_{\lam\in \triangle_{\cX}}  \max_{x\in\cX} \|x\|^2_{(H^\pes_\lam)^{-1}}$ where $H^\pes_\lam = \sum_{x\in\cX} \lam_x \dmu_m^\pes(x) x x^\T$ so that the support size of $\hlam^{\war}$ is at most $d(d+1)/2$.}
    \State Pull arm $x$ exactly $\lcl \hlam^{\war}_x \cd \gam(d) g^{\war}  \rcl$.
    \State{\textbf{Return} the MLE $\hat{\theta}$ computed from these samples.}
  \end{algorithmic}
\end{algorithm}

\begin{remark}\label{rem:gopt}
  Section 3 of \citet{todd16minimum} provides various algorithms for solving the G-optimal design (line~\ref{line:gopt} of Algorithm~\ref{alg:war}).
  For example, using the Kumar-Yildirim initialization \cite{kumar05minimum} along with the Khachiyan first-order algorithm \cite{khachiyan96rounding} results in a 2-approximate solution with the support size of $O(d\ln(\ln(d)))$ (see Lemma 3.7(ii) of \citet{todd16minimum}).
\end{remark}

\WAR enjoys the following correctness guarantee:
\def\low{{\mathsf{low}} }
\begin{thm}\label{thm:warmup-correctness}
  Assume $|\th_*|  \ge 2.399$.
  Suppose $U \le 2.399$.
  Then, with probability at least $1-\dt$, 
  \begin{enumerate}[leftmargin=2em]
    \item[(i)] \WAR is a $\delta$-valid warmup algorithm.
    \item[(ii)] The sample complexity of the Pessimistic planning phase of \WAR is never worse than the naive warmup, i.e., $O(g^{\naive}\gamma(d))$; see~\eqref{eq:warmup-naive}.
    \item[(iii)] For $d=1$, if $ {\cX_L} := \{ x\in\cX: |x\th_*| \le \fr L r\}$ is nonempty, then our algorithm finds a design $\lam$ whose multiplicative approximation ratio w.r.t. the optimal continuous design
    \[
      \min_{\lambda\in\triangle_{[-1,1]}} \max_{x\in \mc{X}} \|x\|^2_{H_\lambda(\theta_*)^{-1}} = \frac{(\th^*)^2}{0.439...} \cd \max_{x\in\cX} |x|
    \]
    is 
    $\fr{1}{0.41} \fr{0.439..}{\dot\mu(x_0 \th_*) (x_0 \th_*)^2} $ where $ {x_0} := \arg \max_{x\in\cX_L} |x\th_*|$.
  \end{enumerate}  
\end{thm}

Theorem~\ref{thm:warmup-correctness}(iii) provides an interesting characterization in 1d of when we can guarantee that the warmup does not scale with $\kappa$ but rather scale with $\|\theta^{\ast}\|^2\approx \log^2(\kappa)$. 
When $L/r = 2$, the approximation ratio is $\approx 18.23$ in the best case of $x_0 = L/r$, and in general it degrades as $|x_0\th_*|$ decreases.
Theorem~\ref{thm:warmup-correctness} reflects the importance of existence of arms with large variances which makes sense given that the concentration bound scales like $\|x\|_{H_t(\th_*)^{-1}}$.
Note that by reducing $r$ or increasing $L$ we can guarantee that $\mc{X}_{L}$ is nonempty, at the cost of increasing the sample complexity.
\begin{thm}\label{thm:warmup-sc}
  Let $ {\Delta_\tw} := \mu(U) - \mu(L)$.
  In \WAR, Pessimistic planning assigns total $O(\gam(d) g^{\war} + d^2)$ samples.
  Furthermore, under the same assumptions as Theorem~\ref{thm:warmup-correctness}  with probability at least $1-2\dt$, when $d=1$ Optimistic probing takes no more than $2 + \log_r(|\th_*|/L)$ iterations and each iteration of optimistic probing takes no more than $O(d\ln(\ln(d)) \cd \Delta_\tw^{-2}\ln(\Delta_\tw^{-2}|\cX|/\dt))$ samples.
\end{thm}

Specifically, a smaller $r$ prolongs the number of optimistic probing iterations, and a larger $L$ increases $\Delta_\tw^{-2}$ so the per-iteration sample complexity increases as well.
One can show that $g^{\war}$ is $O(d)$ ignoring the dependence on $\|\th_*\|$, so the total number of samples assigned by pessimistic planning $O(d^2)$.
Proving the overall sample complexity of \WAR for multi-dimensional cases would likely require analyzing how the volume of the confidence set evolve over optimistic probing iterations; we leave this as future work.

\textbf{Numerical evaluation.}
To verify the performance of WAR numerically, we have drawn 20 arms from the a three-dimensional unit sphere.
The unknown $\th^*$ was drawn the same way but scaled to have the norm $S \in \{2,4,8\}$.
We have run the naive warmup~\eqref{eq:warmup-naive}, WAR~\eqref{alg:war}, and the oracle warmup that solves $ g^{\ast} = \min_{\lam \in \triangle_\cX } \max_{x\in\cX} \|x\|^2_{H_\lam (\th_*)^{-1}}$.
We then computed the total number of samples required to satisfy the warmup condition~\eqref{eq:warmup} from each method, ignoring the integer effect for simplicity.
We repeat this process 5 times and report the result in Table~\ref{tab:numerical} where WAR is significantly better than the naive warmup and not far from the oracle warmup.

\begin{table*}
    \centering
    \begin{tabular}{crrr}\hline
        Warmup & $S=2$ & $S=4$ & $S=8$
    \\\hline  Naive  &  8,377$\pm$\phantom{00}0.3    &  49,794$\pm$\phantom{000}0.1    & 2,623,477$\pm$\phantom{0000}3.0
    \\        WAR    &  6,536$\pm$237.9  &  19,701$\pm$\phantom{0}805.0  & 122,405$\pm$30815.5
    \\        Oracle &  4,970$\pm$\phantom{0}69.0   &  11,720$\pm$1094.2 & 50,258$\pm$\phantom{0}4052.5
    \\\hline
    \end{tabular}
    \caption{Numerical evaluation of the naive warmup, WAR, and the oracle warmup. Each cell contains the average amount of samples required to satisfy the warmup condition and the standard deviation.}
    \label{tab:numerical}
\end{table*}

%
%


\section{On Dimension Dependence in Warmup}
\label{sec:dimension}


In this section we discuss the tightness of the warm-up condition in Theorem 1 of~\citet{jun2020improved} as well as Lemma~\ref{lem:mean_param_bound}. Both confidence widths are tight asymptotically as they match up to constant factors the Gaussian-like deviation observed from the delta method. What is not immediately clear is whether the condition that 
  $\max_{s\in [t]} \| x_s \|^2_{H_t(\th_*)^{-1}}  \lsim \fr{1}{d}  $
is optimal (note we omit the factors involving $\dt$ for simplicity).\footnote{
  When this condition is not true, one can use the confidence bound of~\citet{faury2020improved}, which comes with a factor of $\sqrt{d}S$ in the width (can also be tuned to $\sqrt{dS}$).
}. 
To satisfy this, one needs to pull at least $\Omega(d^2)$ arms even in the best case where $\dmu(x_s^\T \th_*) = \Omega(1)$.
This is quite different from the standard linear model's fixed design inequality (e.g., \citet[Eq (20.2)]{lattimore2020bandit}) that requires us to pull $d$ linearly independent arm.
Requiring only $O(d)$ arm pulls corresponds to relaxing the warmup condition to the following conjectured one: 
\begin{align}\label{eq:conjecture}
    \max_{s\in [t]} \| x_s \|^2_{H_t(\th_*)^{-1}}  \lsim 1 ~.
\end{align}
which is likely necessary. 
Note, if we have to pull $d^2$ arms before using the concentration inequality, we will have to pay the lower order term of $d^2$ in regret, implying that the regret bound is vacuous up to $T=O(d^2)$, the current best known rate for logistic linear bandits. Again this is in contrast with the linear setting where the regret is vacuous up to time $T=O(d)$ only.


We claim that the conjecture~\eqref{eq:conjecture} cannot be true for the MLE.
We provide a sketch of our counterexample.
Let $\th_* = (c, c, \ldots, c), c\in \mathbb{R}$ be the true underlying parameter and assume that $\cX = \{e_1, \ldots, e_d\}$, the canonical basis in $\mathbb{R}^d$, and we sample each arm $N$ times.
Denote by $v(i)$ the $i$-th component of a vector $v$.
In this setup, if $\hth$ is the MLE, then $\hth(i)$ is just the one-dimensional MLE considered for each dimension $i$ independently.
Let $x$ be the target arm.
We are interested in controlling the high probability deviation of $x^\T (\hth - \th_*)$, which we call the \textit{prediction error}, by $O(\|x\|_{H_t(\th_*)^{-1}})$ where $H_t$ is computed by $dN$ samples.
The key observation is that the MLE is biased in GLM's except for special cases like linear models; quantifying and correcting the bias has been studied since \citet{bartlett53approximate}, though the corrections often rely on a plug-in estimate or are asymptotic in nature and do not have precise mathematical guarantees for finite samples.
We now show that the prediction error may be dominated by the bias, rather than the variance, and so we are forced to oversample to correct this bias. 
Specifically, consider the following decomposition of the prediction error in the setting above: 
\begin{align*}
  \underbrace{\sum_{i=1}^d x(i) \cd (\hth(i) - \th_*(i))}_{\textstyle =: \text{(prediction error)}}
  &= \sum_{i=1}^d x(i) \cd\underbrace{ (\EE[\hth(i)] - \th_*(i))}_{\textstyle=: \text{(A)} } \\
  & + \underbrace{\sum_{i=1}^d x(i) \cd (\hth(i) - \EE[\hth(i)])}_{\textstyle=: \text{(B)}}~.
\end{align*}
The bias term (A) is the bias that is incurred per coordinate. In this setting, critically, \textit{the magnitude of the coordinate-wise bias clearly does not depend on the dimension}.\footnote{
In fact, the bias of the MLE is not well-defined since the observations can be all 1s or all 0s from an arm. One can go around it by setting a suitable for those special cases; e.g., when observing all 1s (or all 0s), set $\hth_i = \log(\fr{\hp}{1-\hp} )$ where $\hp=\fr{n-.5}{n} $.
}
By choosing $x = (h_1/\sqrt{d}, \ldots, h_d/\sqrt{d})$ with $h_i := \sign(\EE[\hth(i)] - \th_*(i))$, one can see that the bias term (A) will grow with $d$.
Consequently, even if the deviation (B) is controlled, i.e. (B) is bounded by  $\|x\|_{H_t(\th_*)^{-1}} = \sqrt{\sum_{i=1}^d \fr{(h_i/\sqrt{d})^2}{\dmu(c) N} } = \sqrt{\fr{1}{\dmu(c) N} }$ (which does not grow with $d$), the bias will be the bottleneck in controlling the LHS.
This means that, for large enough $d$, one cannot aim to control the prediction error by $O(\|x\|_{H_t(\th_*)^{-1}})$ unless we choose the number of samples $N$ as a function of $d$.
This is effectively the role of the warmup condition~\eqref{eq:warmup} -- it is requiring an oversampling w.r.t. $d$ so the bias is controlled.

We conjecture that the warmup condition~\eqref{eq:warmup} is tight for a concentration inequality on the MLE.
To explain our reasoning, consider the same setup as above.
Suppose the deviation (B) behaves like $\|x\|_{H_t(\th_*)^{-1}} = \sqrt{\fr{1}{\dmu(c) N} }$.
Using the formula by~\citet{cordeiro91bias}, the bias of order $1/N$ for each coordinate is $\fr{1}{2N} \cd \fr{\mu(c) - \mu(-c)}{\dmu(c)}  $ (through a second order Taylor series expansion); we also confirm it numerically in our supplementary. 
Thus, as long as $c$ is bounded away from 0, we have that the first order bias is $\Theta\del{\fr{1}{N \dmu(c)} }$.
Let us set $N = q \dmu(c)^{-1} $ for some $q \ge 1$.
Then, the bias (A) is $d \cd \fr{1}{\sqrt{d}} \cd\fr{1}{q} = \fr{\sqrt{d}}{q} $ and the deviation term (B) is $\sqrt{\fr{1}{q}} $.
To control the bias term to be below the deviation term, we must have $q \ge d$.
This means that we need to sample at least $d \cd N = d q \dmu(c)^{-1} = d^2 \dmu(c)^{-1}$, which matches the warmup condition~\eqref{eq:warmup}.

Note that the result above is specific to the MLE.
For the special case of the canonical basis arm set, one can use an alternative estimator per coordinate such as the KT estimator $\hth(i) = \log(\fr{H + 1/2}{T + 1/2} )$ where $H$ is the number of successes and $T = N - H$~\cite{krichevsky81theperformance}, which is equivalent to the bias correction method by~\citet{cox89analysis}.
The effect is that the bias now scales like $1/N^2$, which can potentially improve the warmup condition~\eqref{eq:warmup}.
In our supplementary, we empirically verify this and provide a conjecture that an extension of the KT estimator may admit $O(d^{4/3})$ sample complexity for the warmup.


Finally, we emphasize that the fixed design inequalities capture the fundamental aspects of the prediction error \textit{without distributional assumptions on the covariates $x$} besides conditional independence.
The warmup conditions for these inequalities indicating Gaussian tails with the asymptotic variance are closely related to the question of `when the predictive distribution can be approximately Gaussian'.
Yet, we do not know what the fundamental limits are for these warmup conditions beyond the standard linear case nor how to correct the bias of the MLE with precise mathematical guarantees.
Just as our discussion above naturally motivated new estimators, research surrounding these inequalities is likely to impact not only bandits but also design of experiments, uncertainty quantification, and improving prediction accuracy in supervised learning.



\bibliography{aaai22}

\appendix
\onecolumn
\section*{Supplementary Material}

\section{Infrastructure Information}

When running the experiments, we have used MacBook Pro 15'' (2019) A1990. 
This laptop has 2.3GHz 8-core Intel Core i9, Turbo Boost up to 4.8GHz, with the memory of 16GB of 2400MHz DDR4.
The code was developed and run under Python 3.8.8 and the relevant list of software libraries can be found in the code.

\section{Impacts and Limitations}

Finally, we discuss some impacts and limitations of our work.
We remark that our warm-up procedure is useful beyond this work.
For example, other kinds of optimal design problems such as D-optimal design also relies on knowing $\theta^*$ especially w.r.t. the variance of each measurement.
We believe one can use \WAR in the same way as we apply do and prove a formal approximation guarantee when using the plugin estimator returned by \WAR. 
Furthermore, \WAR can be immediately used for the pure exploration logistic bandit problems such as~\citet{jun2020improved}.
The same warmup idea should be applicable to other generalized linear models.
Finally, studying fundamental limits of fixed-design-type inequalities such as asking under what conditions we can get the tight Gaussian-like tail inequalities and understanding the fundamental tradeoff between bias and variance \textit{without distributional assumptions on $x$'s} are interesting open research directions. Finally, as with any work that can be used for myopic optimization in recommender systems, there is a danger of introducing and reinforcing pre-existing biases that can be potentially harmful to society. We encourage practitioners to carefully consider the metrics they are using and build-in safeguards to prevent adverse effects.

\section{Related work}

The problem of regret minimization for logistic linear bandits was first introduced by~\citet{filippi10parametric} where they proposed an optimistic approach for GLMs called \texttt{GLM-UCB} and achieved a regret bound of  $\kappa d\sqrt{T}$; hereafter we omit from regret bounds logarithmic factors of variables other than $K$ and lower-order terms.
A follow-up work by \citet{li2017provably} provided a regret guarantee for finite but changing arm sets, which achieved the regret bound of $\kappa\sqrt{dT\log(K)}$ where $K$ is the number of arms by employing a SupLinRel type algorithm~\cite{auer02using} along with a novel fixed design inequality.
Such a bound is better than that of \texttt{GLM-UCB} as long as $K=O(e^d)$.
The bound of~\citet{li2017provably}, however, came with a large lower-order term and assumes stochastic contexts and a lower bound on the minimum eigenvalue.
We remark that Section 5.1 of~\citet{li2017provably} has a discussion on the dimension dependence that we have found to be misleading because they ignore the dependence on $\sig_0^2$, a bound on the smallest eigenvalue, yet  the dependence on $d$ exists in $\sig_0^2$ as $\sig_0^2 = 1/d$ at best and $\sig_0^2 \le 1/d$ in general.
Beyond UCB or SupLinRel type strategies, Thompson sampling methods \cite{abeille17linear,dong2019performance,kveton17stochastic_arxiv,dumitrascu18pgts} have also been developed for logistic linear bandits, but they all include the factor $\kappa$ on the leading term of the regret bound except for \citet{dong2019performance} that consider the Bayesian framework and \citet{dumitrascu18pgts} that do not have a guarantee.
The seminal work by \citet{faury2020improved} established the first regret bound of $d\sqrt{T}$ that does not have $\kappa$ in the leading term, and a follow-up work by~\citet{abeille2021instance} showed that the bound can be further improved by introducing a factor $\sqrt{\dmu^*}$ in the leading term and proved a matching lower bound for the infinite arm set case.
\citet{jun2020improved} proposed a fixed-design confidence interval that improves upon~\citet{li2017provably} in terms of both $\kappa$ and $d$, which led to $\sqrt{dT\log(K)}$ regret bound.
However, their result inherits the assumptions of stochastic contexts and bounded minimum eigenvalues from~\citet{li2017provably}.
As mentioned in the introduction, our algorithm \texttt{HOMER} achieves the best regret bound both in the leading term and the lower order term as long as $K = O(e^d)$.

Finally we point out that Algorithm~\ref{alg:mix-exp-design-regret} is motivated by similar experimental design procedures for regret minimization from Chapter 20 of~\cite{lattimore2020bandit} and \cite{wagenmaker2021experimental}.


\section{Proofs for Section~\ref{sec:alg}}
\subsection{Confidence widths in mean parameter space}
\begin{proof}[Proof of Lemma~\ref{lem:mean_param_bound}]

Due to the condition \eqref{eq:warmup} on $\xi_t^2$, we can apply~\citet[Theorem 1]{jun2020improved} for each $x\in\cX$ and take a union bound to obtain the following: w.p. at least $1-\dt$, 
\begin{align}\label{eq:mean_param_bound-1}
 \forall x\in\cX,  |x^\T(\hat{\theta}_t - \theta_\ast)| \leq 3.5 \|x\|_{H_t(\theta_\ast)^{-1}} \sqrt{\log(2(2+t_{\eff})|\cX|/\delta)}
\end{align}
The LHS above can be written as $\fr{|\mu(x^\T\hth_t) - \mu(x^\T\th_*)|}{\alpha(x^\T\hth_t,x^\T \th_*))} $.
Rearranging it, we have
\begin{align*}
  |\mu(x^\T\hth_t) - \mu(x^\T\th_*)|
    &\le 3.5 \alpha(x^\T\hth_t,x^\T \th_*)  \|x\|_{H_t(\theta_\ast)^{-1}} \sqrt{\log(2(2+t_{\eff})|\cX|/\delta)}
  \\&\le 3.5 \fr{e^D - 1}{D} \dmu(x^\T\th_*)  \|x\|_{H_t(\theta_\ast)^{-1}} \sqrt{\log(2(2+t_{\eff})|\cX|/\delta)}
\end{align*}
where $D := \max_{x \in \cX} |x^\T(\hat{\theta}_t - {\theta}_\ast)|$.
One can show that our condition on $\xi_t^2$ and \eqref{eq:mean_param_bound-1} implies that $D \le 1$, which leads to $\fr{e^D - 1}{D} \le 6.1$.

\end{proof}

\subsection{Regret bound}



The following lemma on relating variance of two $\theta$'s becomes useful, which is implied directly by the proof of \citet[Lemma 5]{jun2020improved}.
\begin{lem} \label{lem:warmup_guarantee}
If $\th$ satisfies $\max_{s\in[t]} |x_s^\T(\th - \th_*)| \le 1$, then
\begin{align*}
  \frac{1}{3}H_\lambda( \theta_\ast) \preceq H_\lambda( \th) \preceq 3H_\lambda( \theta_\ast)~. 
\end{align*}
Furthermore, for any $x$ with $|x^\T(\th - \th_*)|\le1$,
\begin{align*}
  \fr13 \dmu(x^\T\th_*) \le \dmu(x^\T\th) \le 3 \dmu(x^\T\th_*)
\end{align*}
\end{lem}

In the main algorithm, we used the notation $\lambda_{k,i}$ for the probability mass assigned to the $i$-th arm in the set $\cX$.
Hereafter, we use the notation $\lambda_{k,x}$ for arm $x\in\cX$, which we found to be useful.

\begin{lem}\label{lem:mix_h_vs_normal}
For any $\theta$, $H_{\lambda_k}( \theta)^{-1}\preceq 2\frac{n_k^G + n_k^H}{n_k^G}H_{\lambda_k^G}( \theta)^{-1}$ and $H_{\lambda_k}( \theta)^{-1}\preceq 2\frac{n_k^G + n_k^H}{n_k^H}H_{\lambda_k^H}( \theta)^{-1}$
\end{lem}

\begin{proof}

Let $\lambda_{k,x}$ be the $x$ component of the $k$-th design. For any $\theta$,
\begin{align*}
    H_{\lambda_k}( \theta)& = \sum_{x\in\cX} \lambda_{k,x}\dot{\mu}(x^\T\theta)xx^\T\\
    & = \frac{1}{\sum_{x'}\Tilde{\lambda}_{k,x'} }\sum_{x\in\cX} \max\left\{\frac{n_k^G}{n_k^G + n_k^H}\lambda_{k,x}^G, \frac{n_k^H}{n_k^G + n_k^H}\lambda_{k,x}^H \right\}\dot{\mu}(x^\T\theta)xx^\T \\
    & \succeq \frac{1}{\sum_{x'}\Tilde{\lambda}_{k,x'}}\frac{n_k^G}{n_k^G + n_k^H}\sum_{x\in\cX} \lambda_{k,x}^f \dot{\mu}(x^\T\theta) x x^\T\\
    & = \frac{1}{\sum_{x'}\Tilde{\lambda}_{k,x'}}\frac{n_k^G}{n_k^G + n_k^H} H_{\lambda_k^G}( \theta)~.
\end{align*}
Note that
\begin{align*}
    \sum_{x'}\Tilde{\lambda}_{k, x'} & \leq \max\left\{\frac{n_k^G}{n_k^G + n_k^H}, \frac{n_k^H}{n_k^G + n_k^H} \right\} \sum_{x'}\max\{\lambda_{k,x'}^G, \lambda_{k,x'}^H\}\\
    & \leq \max\left\{\frac{n_k^G}{n_k^G + n_k^H}, \frac{n_k^H}{n_k^G + n_k^H} \right\} \sum_{x'}\lambda_{k,x'}^G +  \lambda_{k,x'}^H \\
    & = 2\max\left\{\frac{n_k^G}{n_k^G + n_k^H}, \frac{n_k^H}{n_k^G + n_k^H} \right\}. 
\end{align*}
Hence, 
\begin{align*}
    H_{\lambda_k}( \theta)
    & \succeq 
    \frac{1}{2}\frac{\frac{n_k^G}{n_k^G + n_k^H}}{\max\left\{\frac{n_k^G}{n_k^G + n_k^H}, \frac{n_k^H}{n_k^G + n_k^H} \right\}}
    H_{\lambda_k^G}( \theta) 
    \succeq \frac{1}{2}\frac{\frac{n_k^G}{n_k^G + n_k^H}}{\frac{n_k^G}{n_k^G + n_k^H} + \frac{n_k^H}{n_k^G + n_k^H} }
    H_{\lambda_k^G}( \theta)
     = \frac{1}{2}\frac{n_k^G}{n_k^G + n_k^H}
    H_{\lambda_k^G}( \theta).
\end{align*}
This implies that
\begin{align*}
     H_{\lambda_k}( \theta)^{-1}\preceq 2\frac{n_k^G + n_k^H}{n_k^G}H_{\lambda_k^G}( \theta)^{-1}.
\end{align*}
The proof for $H_{\lambda_k^H}( \theta)$ follows identically.
\end{proof}

\begin{lem}\label{lem:mix_mean_param_regret_correct}
For any $\delta \leq e^{-1}$, with probability at least $1-\delta$, for all rounds $k\in {\mathbb N}$ such that $n_k^H + n_k^G > r(\eps)$, define the event $\mc{E}_k:= \{x_\ast \in \cX_k \ \text{ and } \ \max_{x\in \cX_k}\mu(x_\ast^\T \theta_\ast) - \mu(x^\T \theta_\ast) \leq 8\cdot 2^{-k}\}$. Define $\mc{E} = \bigcap_{k=1}^\infty \mc{E}_k$. The $\mathbb{P}(\mc{E}) \geq 1-\delta$.  
\end{lem}

\begin{proof}
Let $k$ satisfy $n_k^H + n_k^G > r(\eps)$.
Let $x_1,\ldots,x_{n_k} \in \cX_k$ be the arms pulled in iteration $k$ and define $H_{n_k} (\th_*) = \sum_{s=1}^{n_k} \dmu(x_s^\T \th_*)x_s x_s^\T $.
First note that 
\begin{align*}
    \max_{x\in\cX_k} \|x\|_{H_{n_k}({\theta}_{\ast})^{-1}}^2 & \leq \frac{(1+\epsilon)}{n_k}\max_{x\in\cX_k} \|x\|_{H_{\lambda_k}( {\theta}_{\ast})^{-1}}^2 \\
    & = \frac{(1+\epsilon)}{n_k^G + n_k^H}\max_{x\in\cX_k} \|x\|_{H_{\lambda_k}( {\theta}_{\ast})^{-1}}^2 \\
    & \stackrel{\text{Lemma}~\ref{lem:mix_h_vs_normal}}{\leq} 2\frac{(1+\epsilon)}{n_k^G}\max_{x\in\cX_k} \|x\|_{H_{\lambda_k^G}( {\theta}_{\ast})^{-1}}^2 \\
    & \stackrel{\text{Lemma~\ref{lem:warmup_guarantee}} }{\leq} \frac{6(1+\epsilon)}{n_k^G}\max_{x\in\cX_k} \|x\|_{H_{\lambda_k^G}( \hat{\theta}_{k-1})^{-1}}^2\\
    & \leq \gamma(d, |\cX_k|, \delta_k)^{-1}.
\end{align*}
Hence, the mixed allocation satisfies the warmup condition need in each round.
Similarly, we can show that
\begin{align}\label{eq:n_H_k}
  \max_{x\in\cX_k} \|x\|_{H_{n_k}({\theta}_{\ast})^{-1}}^2 \le \frac{6(1+\epsilon)}{n_k^H}\max_{x\in\cX_k} \|x\|_{H_{\lambda_k^H}( \hat{\theta}_{k-1})^{-1}}^2
\end{align}
By Lemma~\ref{lem:mean_param_bound}, we have, with probability $1-\delta_k$, $\forall x\in\cX_k$,

\begin{align*}
    |\mu(x ^\T \hat{\theta}_k ) - \mu(x_s^\T \theta_\ast )| &\leq 6.1\sqrt{\dot{\mu}(x^\T\theta_\ast)^2\|x\|_{H_{n_k}(\theta_\ast)^{-1}}^2 \log(1/\delta_k)}\\
    & \leq 6.1\sqrt{\frac{(1+\epsilon)\dot{\mu}(x^\T\theta_\ast)^2\|x\|_{H_{\lambda_k}(\theta_\ast)^{-1}}^2 \log(1/\delta_k)}{n_k}}\\
    & \stackrel{\text{\eqref{eq:n_H_k}}}{\leq} 6.1\sqrt{\frac{6(1+\epsilon)\dot{\mu}(x^\T\theta_\ast)^2\|x\|_{H_{\lambda_k^H}(\theta_k)^{-1}}^2 \log(1/\delta_k)}{n_k^H}}\\
    & \stackrel{\text{Lemma~\ref{lem:warmup_guarantee}}}{\leq} 6.1\sqrt{\frac{6(1+\epsilon) (3\dot{\mu}(x^\T\hth_{k-1} ))^2 \cd3\|x\|_{H_{\lambda_k^H}  (\hat\theta_{k-1})^{-1}}^2 \log(1/\delta_k)}{n_k^H}}\\
    & \leq 2^{-k}
\end{align*}

Via a union bound over the rounds, this condition holds in every round. 

First, we show that the best arm, $x_\ast$ is never eliminated. 
Note that $\cX_1 = \cX$. 
Assume the inductive hypothesis that $x_\ast \in \cX_k$. For any suboptimal arm $x \in \cX_k$,
\begin{align*}
    \mu(x^\T \htheta_k) - \mu(x_\ast^\T \htheta_k ) 
    & = \mu(x^\T \htheta_k) - \mu(x^\T \theta_\ast) + \mu(x^\T \theta_\ast) - \mu(x_\ast^\T \theta_\ast) + \mu(x_\ast^\T \theta_\ast) -\mu(x_\ast^\T \htheta_k ) \\
    & < \mu(x^\T \htheta_k) - \mu(x^\T \theta_\ast) + \mu(x_\ast^\T \theta_\ast) -\mu(x_\ast^\T \htheta_k ) \\
    & \le 2\cdot 2^{-k}.
\end{align*}
For $x = x_*$, we trivially have $\mu(x^\T \htheta_k) - \mu(x_\ast^\T \htheta_k ) = 0 < 2\cd2^{-k}  $ .
Thus, $\max_{x\in\cX} \mu(x^\T \htheta_k) - \mu(x_\ast^\T \htheta_k )  < 2\cdot 2^{-k}$ and hence $x_\ast \in \cX_{k+1}$. 

Next, we show that for any $x$ such that $\mu(x_\ast^\T \theta_\ast) - \mu(x^\T \theta_\ast) > 4\cdot 2^{-k}$, $x \not\in \cX_{k+1}$. 
\begin{align*}
    \mu(x_\ast^\T \htheta_k ) - \mu(x^\T \theta_\ast)
    & = \mu(x_\ast^\T \htheta_k )- \mu(x_\ast^\T \theta_\ast) + \mu(x_\ast^\T \theta_\ast) -  \mu(x^\T \theta_\ast) 
    \\& > 2\cdot 2^{-k}
\end{align*}
which implies that $x \not\in \cX_{k+1}$. Taken together, these statements imply that the regret of any arm present in round $k$ is at most $8\cdot 2^{-k}$. 
\end{proof}

\begin{proof}[\textbf{Proof of Theorem~\ref{thm:mix_regret_complexity}}]
Let $T' = T - T_B$
be the number of samples taken after the warmup procedure. We assume $T' > 0$. Since all rewards are bounded by $1$, the warmup contributes at most
$R_B$. 
to the total regret. We proceed by bounding the remaining $T'$ samples following warmup. Let $n_{k,x}$ denote the number of times $x$ is pulled in round $k$ and let $\lambda_{k,x}$ denote the fraction of the allocation placed on arm $x$. Assume the event $\cE$ defined in Lemma~\ref{lem:mix_mean_param_regret_correct} that occurs with probability at least $1-\delta$ and that the warmup procedure succeeds that also occurs with probability at least $1-\delta$. 
Define $L := \lceil \log_2(8 (\Delta \vee \nu)^{-1})\rceil$, $\Delta_x := \mu(x_*^\T\th_*) - \mu(x^\T\th_*)$, and $\Delta := \min_{x\in \mc{X} \sm \{x_*\}} \Delta_x$. 
Then, for any $\nu\geq0$, 
\begin{align*}
    R_{T} - R_{B} 
    & = \sum_{x \in \cX \backslash \{x_\ast\}} \Delta_x T_x \\
    & \leq T'\nu + \sum_{k=1}^{L}  \sum_{x \in \cX \backslash \{x_*\}: \Delta_x > \nu} \Delta_x n_{k,x} \\
    & \le T'\nu + \sum_{k=1}^{L}  \sum_{x \in \cX \backslash \{x_*\}: \Delta_x > \nu} \one\cbr{n_k \le r(\eps)} 1\cd n_{k,x} + \one\cbr{n_k > r(\eps)} 8 \cd 2^{-k} n_{k,x} \\
    & \stackrel{\text{Lemma~\ref{lem:mix_mean_param_regret_correct}}}{\leq} T'\nu + L r(\epsilon) + 8 \sum_{k=1}^{L}\one\cbr{n_k > r(\eps)} 2^{-k} n_k \\
    & \leq T'\nu + L r(\epsilon)+ \sum_{k=1}^{L} \one\cbr{n_k > r(\eps)} 8\cdot2^{-k} n_k^G  + \sum_{k=1}^{L} \one\cbr{n_k > r(\eps)} 8\cdot2^{-k} n_k^H
\end{align*}
We analyze each sum individually. 
\begin{align*}
    \sum_{k=1}^{L} 2^{-k} n_k^G 
    &\le L +  \sum_{k=1}^{L} 6(1+\epsilon) 2^{-k}\gamma(d, |\cX_k|, \delta_k) \hat g_k(\lambda_k^G) \\
    & \leq L + 6(1+\epsilon)\gamma(d, |\cX|, \delta_L) \sum_{k=1}^{L} 2^{-k} \max_{x\in\cX_k} \|x\|_{H_{\lambda_k^G}( \hat{\theta}_{k-1})^{-1}}^2 \\
    & = L + 6(1+\epsilon)\gamma(d, |\cX|, \delta_L) \sum_{k=1}^{L} 2^{-k} \min_{\lambda\in \triangle_{\mc{X}_k}}\max_{x\in\cX_k} \|x\|_{H_{\lambda}( \hat{\theta}_{k-1})^{-1}}^2 \\
    & \stackrel{\text{Lemma~\ref{lem:warmup_guarantee}}}{\leq} L + 18(1+\epsilon)\gamma(d, |\cX|, \delta_L)\sum_{k=1}^{L} 2^{-k} \min_{\lambda\in \triangle_{\mc{X}_k}}\max_{x\in\cX_k} \|x\|_{H_{\lambda}( {\theta}_\ast)^{-1}}^2 \\
    & \stackrel{(a)}{\leq} L + 18(1+\epsilon)\gamma(d, |\cX|, \delta_L)\sum_{k=1}^{L} 2^{-k} \min_{\lambda\in \triangle_{\mc{X}_k}}\max_{x\in\cX_k} \kappa_0\|x\|_{(\sum_{x\in \mc{X}_k} \lambda_x xx^{\top})^{-1}}^2 \\
    & \stackrel{(b)}{\leq} L + 18\kappa_0 (1+\epsilon)\gamma(d, |\cX|, \delta_L)  d \\
\end{align*}
where $(a)$ uses the fact that for any $\lambda\in \triangle_{\mc{X}_{k}}$, $H_\lambda(\th_*) \succeq \min_{x\in \mc{X}} \dot\mu(x^{\top} \theta_{\ast}) \sum_{x\in \mc{X}} \lambda_{x} xx^{\top}.$ and $(b)$ uses the Kiefer-Wolfowitz theorem.
Next, 
\begin{align*}
    \sum_{k=1}^{L} 2^{-k} n_k^H
    & \le L + 6(1+\epsilon) 6.1^2 3^3\sum_{k=1}^{L}  2^{k}\hat h_k(\lambda_k^H)\log(\delta_k^{-1})\\
    & = L + 6(1+\epsilon)6.1^2 3^3\sum_{k=1}^{L}  2^{k} \max_{x\in\cX_k} \dot{\mu}(x^\T\hat{\theta}_{k-1})^2\|x\|_{H_{\lam_k^H}( \hat{\theta}_{k-1})^{-1}}^2\log(\delta_k^{-1})\\
    & = L + 6(1+\epsilon)6.1^2 3^3\sum_{k=1}^{L}  2^{k} \min_{\lambda\in \triangle_{\mc{X}_k}}\max_{x\in\cX_k} \dot{\mu}(x^\T\hat{\theta}_{k-1})^2\|x\|_{H_{\lam}( \hat{\theta}_{k-1})^{-1}}^2\log(\delta_k^{-1})\\
    &\stackrel{\text{Lemma~\ref{lem:warmup_guarantee}}}{\leq}  L +18(1+\epsilon)6.1^2 3^6\sum_{k=1}^{L}  2^{k}\min_{\lambda\in \triangle_{\mc{X}_k}}\max_{x\in\cX_k} \dot{\mu}(x^\T{\theta}_{\ast})\|\sqrt{\dot{\mu}(x^\T{\theta}_{\ast})} x\|_{H_{\lambda}( {\theta}_\ast)^{-1}}^2\log(\delta_k^{-1})\\
    &\leq  L +18(1+\epsilon)6.1^2 3^6\sum_{k=1}^{L}  2^{k}\max_{x\in\cX_k} d\dot{\mu}(x^\T{\theta}_{\ast})\log(\delta_k^{-1})~.
\end{align*}
where for the last line we have again used the Kiefer-Wolfowitz theorem.  Recall that $\cX_k\subset \cS_k = \{x \in \cX: \mu(x_\ast^\T \theta_\ast) - \mu(x^\T \theta_\ast) \leq 8 \cdot 2^{-k}\}$. Let $\mu^\ast := \mu(x_\ast^\T \theta_\ast)$. 
Hence, 
\begin{align*}
    \max_{x\in\cS_k} \dot{\mu}(x^\T{\theta}_{\ast}) 
    & = \max_{x\in\cS_k} \mu(x^\T{\theta}_{\ast}) (1 - \mu(x^\T{\theta}_{\ast})) \\
    & \leq \max_{x\in\cS_k} \mu^\ast \cdot (1 - \mu(x^\T{\theta}_{\ast})) \\
    & \leq \mu^\ast \cd (1 - \mu^\ast) + 8\mu^\ast2^{-k} \\
    & \leq \dot\mu^\ast + 8\cdot2^{-k}.
\end{align*}
Plugging this in, we have that
\begin{align*}
    2^{k}\max_{x\in\cS_k} d \dot{\mu}(x^\T{\theta}_{\ast}) 
    \leq 2^{k}d \dot\mu^\ast + 8d.
\end{align*}
Summing up over all rounds we see that 
\begin{align*}
    \sum_{k=1}^{L} 2^{k} \max_{x\in\cX_k} d \dot{\mu}(x^\T{\theta}_{\ast})
    &\leq  \sum_{k=1}^{\lceil \log_2(8 (\Delta \vee \nu)^{-1})\rceil} \del{2^{k}d \dot\mu^\ast + 8d}\\
    &\leq 8d\lceil \log_2(8 (\Delta \vee \nu)^{-1})\rceil + 32d \dot\mu^\ast \frac{1}{\Delta \vee \nu}
\end{align*}

Thus, for some unspecfied constant $c$, the total regret of our algorithm is  
\begin{align*}
    R_{T_B} + c\del{T'\nu + (1+\epsilon)\log(\delta_L^{-1}) d\dot\mu^\ast (\Delta \vee \nu)^{-1} + (1+r(\eps))\log((\Delta\vee \nu)^{-1}) + (1+\epsilon)d\kappa_0\gamma(d, |\cX|, \delta_L)}~.
\end{align*}
\end{proof}

\begin{proof}[\textbf{Proof of Corollary~\ref{cor:minimax_regret}}]

Throughout, we take $\delta = 1/T$. Hence with probability at most $1/T$ we get regret bounded by $T$. Otherwise, with probability at least $1-1/T$, we get the bound given in Theorem~\ref{thm:mix_regret_complexity}. Combining these with the law of total expectation bounds the regret in expectation. Additionally, we make use of the naive warmup~\eqref{eq:warmup-naive} which suffers regret at most $O(d^2\kappa \log(|\cX|T))$ where we note that $\kappa \geq \kappa_0$.

Note that Theorem~\ref{thm:mix_regret_complexity} holds for any $\nu \geq 0$. Therefore, we may tune $\nu$. 

First, plugging in $\nu = O\left(\sqrt{\frac{\log\left(|\cX|/\delta\right)d \dot\mu(x_\ast^\T\theta_\ast)}{T}} \right)$ leads to a regret within a doubly logarithmic factor of
\begin{align*}
    \sqrt{dT\dot\mu(x_\ast^\T\theta_\ast)\log(|\cX|T)} + d \log\left(\sqrt{\frac{T}{d\dot\mu(x_\ast^\T\theta_\ast)}}\right) + d^2\kappa\log(|\cX|T) + d\kappa_0\del{d + \ln\del{|\cX|T}}
\end{align*}
which is on the order of 
\begin{align*}
    \sqrt{dT\dot\mu(x_\ast^\T\theta_\ast)\log(|\cX|T)} + d^2\kappa\log(|\cX|T)
\end{align*}
since $\kappa_0 \leq \kappa$ and $1/\dot\mu(x_\ast^\T\theta_\ast) \leq \kappa$. 

Otherwise, we may set $\nu = 0$ leading to a regret within a doubly logarithmic factor of 
\begin{align*}
    \frac{d\dot\mu(x_\ast^\T\theta_\ast)}{\Delta}\log(|\cX|T) + d \log\left(\frac{1}{\Delta}\right) + d^2\kappa\log(|\cX|T) + d\kappa_0\del{d + \ln\del{|\cX|T}}
\end{align*}
which is on the order of 
\begin{align*}
    \frac{d\dot\mu(x_\ast^\T\theta_\ast)}{\Delta}\log(|\cX|T) + d^2\kappa\log(|\cX|T). 
\end{align*}
Combining these two statements completes the proof.
\end{proof}

\section{More on Section~\ref{sec:warmup}}

\subsection{A Practical Version of \WAR}

For clarity, we denote by $x^{(i)}$ the $i$-th arm in the arm set $\cX$ and define $K := |\cX|$.
For completeness, Algorithm~\ref{alg:2approx}, which invokes Algorithm~\ref{alg:initial-support} and~\ref{alg:wa}, describes the pseudocode for finding a 2-approximate solution for line~\ref{line:2approx} in Algorithm~\ref{alg:war}.
For solving the $G$-optimal design problem in pessimistic planning, one can use Frank-Wolfe with the standard step size: for iteration $k$, set $\lambda \larrow (1-\alpha_k)\lambda + \alpha_k e_{j^*}$ where $\alpha_k = 2/(k+2)$ and $j^* = \arg \min_{j\in[K]} \nabla_{\lam_j}  \del{ \max_{i\in[K]} \|x^{(i)}\|^2_{(H^\pes_\lambda)^{-1}}}$.

For optimistic probing of WAR, notice that if an arm $\Supp(\hlam^{(m)})$ is accepted already, then we do not need to pull that arm, which saves the overall sample complexity.
To encourage this, when finding the initial support (Algorithm~\ref{alg:initial-support}) one can replace the variable $\ell$ therein with the index of the accepted arms that are not included to $\cV$ yet.

Python implementation of WAR can be found in our supplementary material.

\begin{algorithm}
  \caption{Find 2-approximate solution}
  \label{alg:2approx}
  \small
  \begin{algorithmic}
    \Require{A set of arms $\cX=\{x^{(1)},\ldots,x^{(K)}\} \subset \RR^d$ with cardinality $K$.}
    \State Invoke Algorithm~\ref{alg:initial-support} to obtain a set of initial vectors $\cV \subseteq \cX$.
    \State Let $\lambda \in\triangle_{[K]}$ such that $u_k  = 1/|\cV|$ if $k \in \cV$ and $u_k=0$ otherwise.
    \State Invoke Algorithm~\ref{alg:wa} with $\lambda$ and $\eps \larrow 1$ to obtain a $2$-approximate solution $\hat \lambda$. 
    \State \textbf{return} $\hat \lambda$
  \end{algorithmic}
\end{algorithm}
\begin{algorithm}
  \caption{Initial support~\cite{todd16minimum}}
  \label{alg:initial-support}
    \small
  \begin{algorithmic}
    \Require{A set of arms $\cX=\{x^{(1)},\ldots,x^{(K)}\} \subset \RR^d$ with $|\cX| = K$ such that $\cX$ spans $\RR^d$}
    \State $Q \larrow I \in \RR^{d\times d}$
    \State $v \larrow Q_{\cdot,1} \in \RR^{d}$
    \State $\cV \larrow \emptyset$
    \For {$j=1,\ldots,d$}
        \State $\ell \larrow \arg\max_{k\in[K]} |v^\T x^{(k)}|$
        \State $y \larrow x^{(\ell)}$
        \State $\cV \larrow \cV \cup \{y\}$
        \State $w \larrow Q^\T y$ 
        \If {$j>1$}
            \State $w(i) \larrow  0, \forall i \le j-1$
        \EndIf
        \State $Q \larrow Q - Q (w + s_j \|w\| e_j ) \fr{1}{\|w\|(|w_j| + \|w\|)} (w + s_j\|w\| e_j)^\T$ where $s_j = 2\one\{w(j) \ge 0\} - 1$ and $e_j$ is $j$-th indicator vector.
        \State $v \larrow Q_{\cdot,j+1}$
    \EndFor
    \State \textbf{return} $\cV$
  \end{algorithmic}
\end{algorithm}
\begin{algorithm}
  \caption{Wolfe's algorithm with away step~\cite{todd16minimum}}
  \label{alg:wa}
    \small
  \begin{algorithmic}
    \Require{The arm set $\cX=\{x^{(1)},\ldots,x^{(K)}\} \subset \RR^d$, an initial point $\lambda \in \triangle_{[K]}$ and $\eps>0$.}
    \While {True}
        \State Compute (scaled) Cholesky factorization of $V := X\diag(\lambda)X^\T$ where $X\in\RR^{K\times d}$ is the design matrix of the arm vectors.
        \State $\om \larrow [(x^{(i)})^\T V^{-1} x^{(i)}]_{i\in[K]}$ and a (scaled) 
        \State $\eps_+ \larrow \max_{k\in[K]} (\om(k) - d)/d$ and let $i$ be its arg max.
        \State $\eps_- \larrow \max_{k\in[K]: \lambda(k) > 0} (d-\om(k))/d $ and let $j$ be its arg max.
        \If {$\max\{\eps_+,\eps_-\}\le\eps$}
            \State \textbf{break}
        \EndIf
        \If {$\eps_+ > \eps_-$}
            \State $\psi^* \larrow \frac{\om(i) - d}{(d-1) \om(i)}$.
            \State $\lambda \larrow (1+\psi^*)^{-1}(\lambda + \psi^* e_i)$.
        \Else 
            \State $\psi^* \larrow \frac{\om(j) - d}{(d-1) \om(j)}$.
            \State $\psi \larrow \max\{-\lambda(j),\psi^*\}$.
            \State $\lambda \larrow (1+\psi)^{-1}(\lambda + \psi e_j)$.
        \EndIf
    \EndWhile
    \State \textbf{return} 
  \end{algorithmic}
\end{algorithm}

\subsection{Proofs}


\begin{proof}[\textbf{Proof of Theorem~\ref{thm:warmup-correctness}}]
  The first claim $(i)$ can be shown by the fact that the estimate $\dmu^\pes(x)$ is a high probability lower bound on $\dmu(x^\T\th^*)$.
  To prove $(ii)$, we remark that the sample assignments $N^\naive_x$ is equivalent to solving the $H^\pes$-optimal design problem with $\dmu^{\pes}(x) = \kappa$ for all $x\in\cX$ and then, denoting by $\hlam^\naive$ its solution, assigning $\lcl \hlam_x^\naive \gam(d)\cd d\rcl$ for each arm $x\in\cX$ since $\min_\lam \max_{x\in\cX} \|x\|^2_{V_\lam^{-1}} = d$ by the Kiefer-Wolfowitz Theorem~\cite{kiefer60theequivalence}.
  Since our algorithm's variance estimates satisfy $\dmu^{\pes}(x) \ge \kappa$, we never spend more samples in total, up to an $O(d^2)$ additive term due to the rounding.
  
  We now prove $(iii)$.
  In 1d, one can show that the algorithm accepts at most one arm.
  We first claim that when $\cX_L$ is nonempty, the algorithm accepts an arm $x_a$ such that 
  \begin{align}\label{eq:acceptedarm}
    |x_a \th^*| \ge  |x_0\th^*|
  \end{align}
  (This is equivalent to $|x_a| \ge |x_0|$.)
  We prove the claim by contradiction: Suppose $|x_a \th^*| <  |x_0\th^*|$.
  Since the arm $x_0$ must be accepted when pulled, it must have been true that $x_0$ was not pulled due to being removed by the condition $|x_0| > |x'|/r$ for some rejected arm $x'$.
  This implies that $|x_0 \th^*| > |x' \th^*|/r \ge L/r$, where the last inequality uses the fact that $x'$ was rejected.
  This contradicts $x_0 \in \cX_L$.
  
  Next, we show that the final solution of the $H^\pes$-optimal design problem is as good as placing all the probability mass on the arm that we have accepted.
  Note that the pessimistic estimate of $|\th^*|$ is $\fr{U}{|x_a|}$. 
  Thus, the optimization problem we solve is
  \begin{align*}
    \hx = \arg \max_{x\in\cX} \dot\mu\del{x \fr{U}{x_a} } x^2 ~.
  \end{align*}
  This means that we have
  \begin{align*}
    \dot\mu(\hx \th^*) \hx^2
    \ge  \dmu(\hx \fr{U}{x_a} )\hx^2
    &\ge \dmu(x_a \fr{U}{x_a} )x_a^2
    \\&= \dot\mu(x_a \th^*) x_a ^2 \cd \fr{\dot\mu(U)}{\dot\mu(x_a \th^*)} 
    \\&\ge \dot\mu(x_a \th^*) x_a^2 \cd \fr{\dot\mu(U)}{1/4} \tag{$\dmu(z) \le 1/4, \forall z$ }
    \\&\ge 0.41 \cd  \dot\mu(x_a \th^*) x_a^2  \tag{$U \le 2.399$ } 
    \\&\sr{(a)}{\ge} 0.41 \cd  \dot\mu(x_0 \th^*) x_0^2 
    \\&= 0.41 \cd \fr{ \dot\mu(x_0 \th^*) x_0^2 }{\max_{x \in \in[-1,1]}  \dmu(x\th^*) (x\th^*) ^2} \max_{x \in \in[-1,1]}  \dmu(x\th^*) (x\th^*)^2
    \\&= 0.41 \cd \fr{ \dot\mu(x_0 \th^*) (x_0\th^*)^2 }{0.439..} \max_{x \in \in[-1,1]}  \dmu(x\th^*) x^2
  \end{align*}
  where $(a)$ is by the fact that when $X \le 2.399$ the function $\dot\mu(X) X^2$ is increasing and thus $|x_a|\ge|x_0| \implies \dot\mu(x_a\th^*) (x_a\th^*)^2 \ge \dot\mu(x_0\th^*) (x_0\th^*)^2$.
  
  This proves the statement (iii).
\end{proof}

\begin{proof}[\textbf{Proof of Theorem~\ref{thm:warmup-sc}}]
  First,   the first statement is trivial given that the support size of $\lambda$ is $d(d+1)/2)$.

  To prove $(i)$, we show that the number of iterations are at most $2 + \log_r(|\th^*|/L)$.
  Assume the concentration event $\cE$.
  Suppose that the loop has terminated after $k$ iterations.
  It suffices to consider the case where we have accepted an arm in the last iteration since the case where we never accept an arm can be made more difficult by adding more arms.
  Let $x_j$ be the arm that was tested at $j$-th iteration.
  Since $x_{k-1}$ was rejected, we have $|x_{k-1}| \ge L/|\th^*|$.
  Then, 
  \begin{align*}
    \fr{L}{|\th^*|} \le |x_{k-1}| < \fr{|x_{k-2}|}{r}  < ... < \fr{|x_1| }{r^{k-2}} \le \fr{1}{r^{k-2}} ~.
  \end{align*}
  This implies that $k-2 \le \log_r(|\th^*|/L)$.
  
  To prove $(ii)$, we need to bound the number of samples spent on the accept-reject procedure on an arm $x$.
  Let us use the shortcut $\mu_x = \mu(x\th^*)$.
  Using symmetry, we safely assume $\mu_x < 1/2$ without loss of generality.
  Let us define $\ell:= \mu(-L)$ and $u := \mu(-U)$.
  When the arm $x$ is not accepted yet, we have
  \begin{align}\label{eq:notaccepted}
    \hmu_x - \sqrt{\fr{\hmu_x(1-\hmu_x) 2\ln(3/\dt_N)}{N}} - \fr{3\ln(3/\dt_N)}{N} \le u~.
  \end{align}
  Throughout the proof, we assume the concentration event $\cE$.
  We also will use the standard Bernstein's inequality with the union bound over all samples size and the arm set $\cX$:
  \begin{align}\label{eq:standard-bernstein}
    |\hmu_x - \mu_x| \le 2 \sqrt{\mu_x(1-\mu_x) \ln(2 /\dt_N)} + \ln(2/\dt_N)
  \end{align}
  that can be shown by the standard Chernoff technique along with the well-known upper bound on the moment generating function of a centered random variable $\eps$ such that $|\eps|\le1$ w.p. 1 (e.g., Lemma 7 of~\citet{faury2020improved}).

  Let us omit the subscript $x$ from now on and introduce $c = \ln(3/\dt)$. 
  Suppose the sample of $x$ is not terminated yet, i.e., $x$ is not accepted nor rejected yet.
  We consider two cases: $\mu \ge \fr{\ell + u}{2}$ and $\mu < \fr{\ell+u}{2}$.
  
  \textbf{Case 1. } $\mu \ge \fr{\ell + u}{2}$.\\
  Using the fact that arm $x$ is not accepted yet, from~\eqref{eq:notaccepted}, we have
  \begin{align*}
    - 3 \fr{c}{N}  -u
    &\le - \hmu + \sqrt{\fr{\hmu(1-\hmu)2c}{N} }
    \\&\le - \hmu + \sqrt{\fr{\hmu 2c}{N} }
    \\&\le - \del{1-\fr1a}\hmu + \fr{a}{2}\fr{c}{N}  \tag{for any $a>0$ by Fenchel-Young Ineq.}
    \\&\le - \del{1-\fr1a}\del{\mu - 2 \sqrt{\mu \fr{c}{N} } - \fr{c}{N} }  + \fr{a}{2}\fr{c}{N}  \tag{assume $a > 1$; use \eqref{eq:standard-bernstein}}
    \\&\le - \del{1-\fr1a}\del{ \del{1-\fr{1}{a}} \mu + (-1-a)\fr{c}{N} }  + \fr{a}{2}\fr{c}{N}  
    \\&= - \del{1-\fr1a}^2 \mu + \del{\fr{3}{2}a - \fr{1}{a} }\fr{c}{N}
    \\ \implies  N &\le \fr{\del{\fr{3}{2}a - \fr{1}{a} + 3 }c}{u - \del{1-\fr1a}^2\mu} 
  \end{align*}
  We now choose $a$ such that $u - \del{1-\fr{1}{a} }^2\mu = \fr{\mu - u}{2} $ which is $a = \fr{1}{1-\sqrt{\fr{3}{2}\fr{u}{\mu}-\fr12}} > 1$.
  Then,
  \begin{align*}
    N \le \fr{\del{3a - \fr{2}{a} + 6 }c}{\mu - u} ~.
  \end{align*}
  Using $\mu \ge \fr{\ell+u}{2} $, we have $\mu-u\ge (\ell-u)/2$ and $\fr{u}{\mu}  \le 1$ (note $u > \ell$).
  With algebra, we have
  \begin{align*}
    N \le \del{ \fr{8}{\ell - u} + 6 }\fr{2c}{\ell-u} ~.
  \end{align*}
  
  \textbf{Case 2.} $ \mu < \fr{\ell+u}{2}$ \\
  Since the arm $x$ is rejected yet, we have
  \begin{align*}
    \hmu + \sqrt{\fr{\hmu(1-\hmu) 2c}{N}} + \fr{3c}{N} \ge  {\ell}~.
  \end{align*}
  Then,
  \begin{align*}
    \ell
    &\le \hmu + \sqrt{\fr{\hmu(1-\hmu) 2c}{N}} + \fr{3c}{N}
    \\&\le \hmu + \sqrt{\fr{\hmu 2c}{N}} + \fr{3c}{N}
    \\&\le \del{1+\fr{1}{b}} \hmu + \del{\fr b 2 +3} \fr{c}{N}      \tag{for any $b>0$ }
    \\&\le \del{1+\fr{1}{b}} \del{\mu + 2\sqrt{\fr{\mu}{N} c } + \fr{c}{N} }  + \del{\fr b 2+3} \fr{c}{N} \tag{Use~\eqref{eq:standard-bernstein}}
    \\&\le \del{1+\fr{1}{b}} \del{\del{1+\fr{1}{b}}\mu + (1+b) \fr{c}{N} }  + \del{\fr b 2+3} \fr{c}{N} 
    \\&\le \del{1+\fr{1}{b}}^2\mu  + \del{\fr32 b+\fr{1}{b}+4} \fr{c}{N} 
    \\\implies   N &\le \min_{b>0} \fr{(\fr32 b + \fr{1}{b} + 4 )c}{\ell - \del{1+\fr{1}{b} }^2\mu} ~.
  \end{align*}
  For the case where $\mu < \ell/8$, we can choose $b=1$ and obtain
  \begin{align*}
    N \le \fr{\fr{13}{2} c}{\ell - 4\cd \fr{\ell}{8} }  = 13c~.
  \end{align*}
  When $\ell/8 \le \mu \le \fr{\ell+u}{2}  $, we choose $b$ such that $\ell - \del{1+\fr{1}{b}}^2\mu = \fr12 (\ell - \mu)$, which is $b = \fr{1}{\sqrt{\fr12 + \fr12 \fr\ell\mu} - 1} $.
  Then,
  \begin{align*}
    N \le \fr{(3 b + \fr{2}{b} + 8 )c}{\ell - \mu}~.
  \end{align*}
  One can show that $b$ is in fact of order $1/(\ell - \mu)$, which means the RHS above is like $O(\fr{1}{(\ell-\mu)^2} )$.
  Since $1/b \le \sqrt{4.5} - 1$, $b = \fr{\sqrt{\fr12 + \fr12 \fr \ell \mu} + 1}{\fr12 (\fr \ell \mu - 1)} \le  \fr{\sqrt{\fr12 + \fr12\cd 8} + 1}{\fr12\del{\fr{2\ell}{\ell+u}}-1} \le (\sqrt{4.5}+1) \fr{2(\ell+u)}{\ell-u} $, and $\ell-\mu \ge \fr12 (\ell-u)$, we have
  \begin{align*}
    N \le \del{\fr{19(\ell+u)}{\ell-u} + 3 + 8   } \fr{2c}{\ell-u} ~.
  \end{align*}
  In both cases, we have
  \begin{align*}
    N \le \del{\fr{19}{\ell-u} + 11} \fr{2c}{\ell-u} ~.
  \end{align*}
  
  Altogether, we have shown that when the arm is not rejected nor accepted, the following holds:
  \begin{align*}
    N \le C_0\fr{c}{(\ell-u)^2}
  \end{align*}
  for some absolute constant $C_0$.
  However, note that we must use a union bounds over the sample count $N$ and the arm set $\cX$.
  Thus, we set $c = \ln(3|\cX|N(N+1)/\dt)$.
  This, however, makes the inequality above implicit.
  Letting $A = C_0/(\ell-u)^2$ and $B =3|\cX|/\dt$, 
  \begin{align*}
    N &\le 2A \ln(N) + A \ln(B)
    \\&= 2A \ln( \fr{N}{4A} \cd 4A) + A \ln(B)
    \\&\le 2A \del{\fr{N}{4A} + \ln(4A)} + A \ln(B)
    \\  \implies N &\le  4A\ln(4A) + 2A \ln(B) ~.
  \end{align*}
  Since the RHS above is the bound that $N$ cannot go above without accepting or rejecting the arm, we can deduce that the number of samples spent after being accepted or rejected is at most $1 + 4A\ln(4A) + 2A \ln(B)  =O\del{(\ell-u)^{-2}  \ln\del{\fr{|\cX| (\ell-u)^{-2}}{\dt} }} $.
  Since we rely on both the empirical Bernstein and the standard Bernstein inequalities to be true, our claim holds with probability at least $1-2\dt$.
  
\end{proof}

\section{On the KT Estimator}

We empirically verify the biases of the MLE and KT estimator in Figure~\ref{fig:bias-scale-mle}.

\begin{figure}[thb]
  \centering
  \begin{tabular}{cc}
    \includegraphics[width=0.4\linewidth]{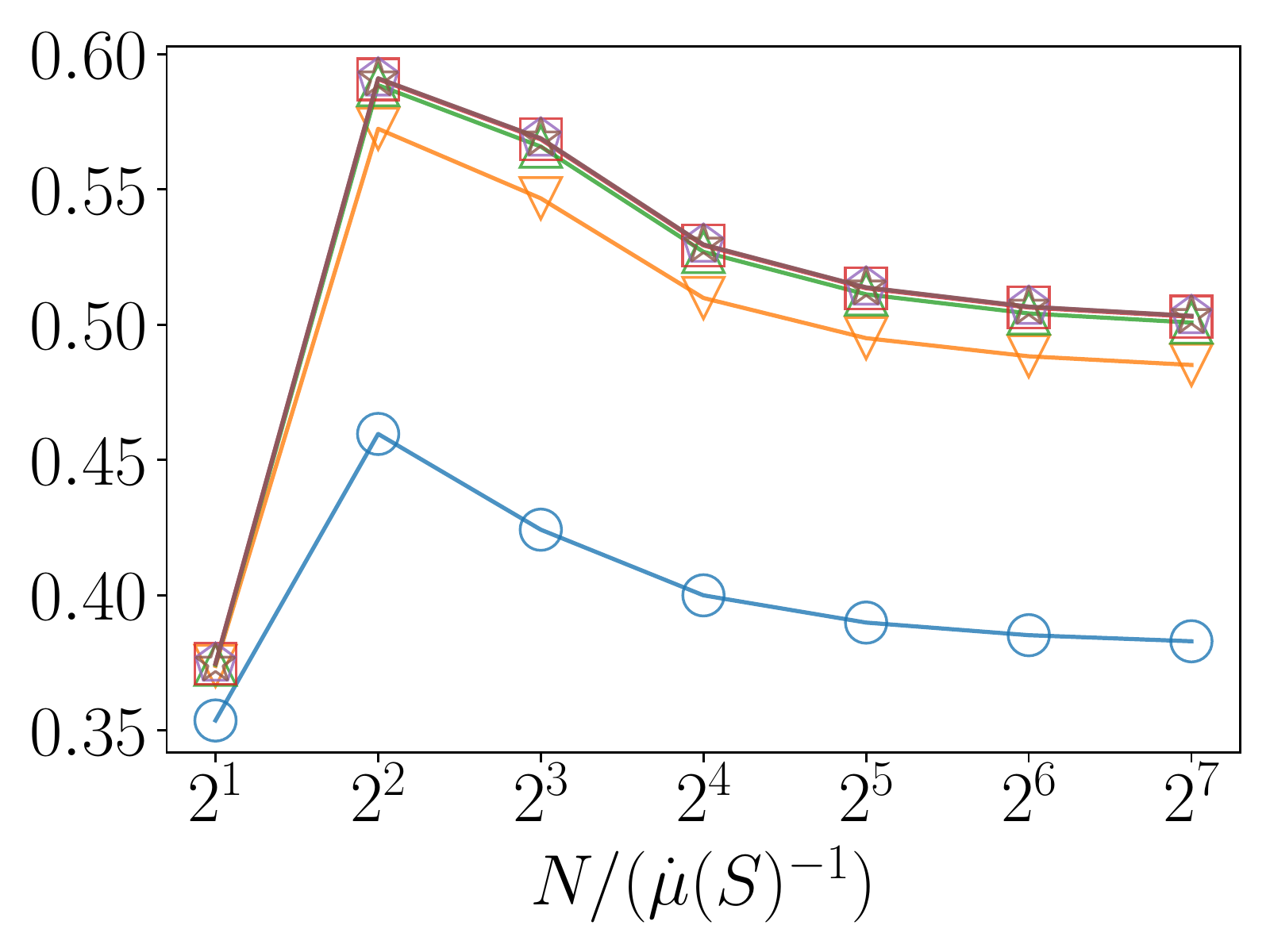}
    & \includegraphics[width=0.4\linewidth]{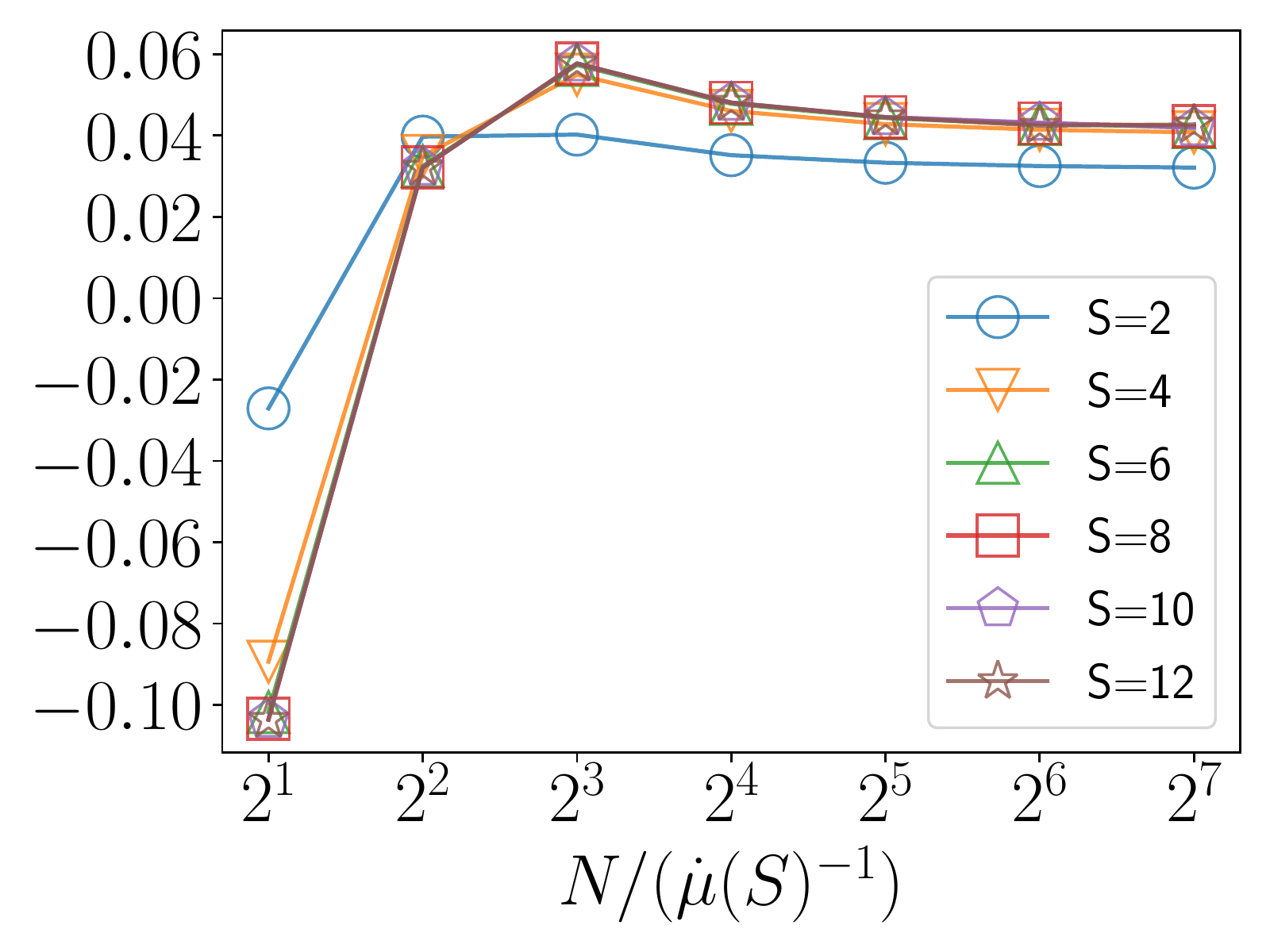}\\
    (a) & (b)
  \end{tabular}
\caption{Numerical verification of the bias of (a) MLE and (b) KT estimator for Bernoulli($\mu(S)$) with $N$ samples that behave like $1/(N\dmu(S))$ and $1/(N\dmu(S))^2$ respectively. 
Details: We consider $N \in \{2,2^2,\ldots,2^7\} \cd \dmu(S)^{-1}$, compute the expectation of the natural parameter estimator minus $S$ using the probability mass function, and then divide this quantity by $\fr{1}{N\dmu(S)}$ for the MLE and $\fr{1}{(N\dmu(S))^2}$ for the KT estimator. 
When $N$ is large enough, the curve flattens out, which confirms the conjectured rates of bias decay.
When the MLE is $\infty$ or $-\infty$, we replace it by $\log(\fr{\hp}{1-\hp} )$ with $\hp=\fr{n-.5}{n} $ or $\hp=\fr{.5}{n}$ respectively to avoid the bias being undefined.
Remark: we have set the sample size $N \gsim \dmu(S)^{-1}$ as this is required for the lower order term of Bernstein's inequality to become non-dominant. 
}
\label{fig:bias-scale-mle}
\end{figure}

Indeed, the bias of the KT estimator scales like $\fr{1}{(N\dmu)^2}$.
Assuming this rate, one can use an identical argument to Section~\ref{sec:dimension} to compute that $N = \Omega(d^{1/3})$ is sufficient to escape the bias-dominating regime, which would imply that the total warmup sample complexity can be $O(d^{4/3})$.
To enjoy small bias beyond the canonical basis, we must define the right multi-dimensional extension of the KT estimator.
The right extension seems to be the regularized MLE with Jeffrey's prior~\cite{firth93bias}.
Developing the fixed-design concentration inequality for such a regularized estimator with a tighter warmup condition dependence on $d$ will require different techniques from~\citet[Theorem 1]{jun2020improved}, and we leave it as future work.

\section{Comparison of Experimental Designs}
In the following we compare $G$ and $H$ optimal designs in an example. 
\begin{figure}
    \centering
    \includegraphics[width=.75\linewidth]{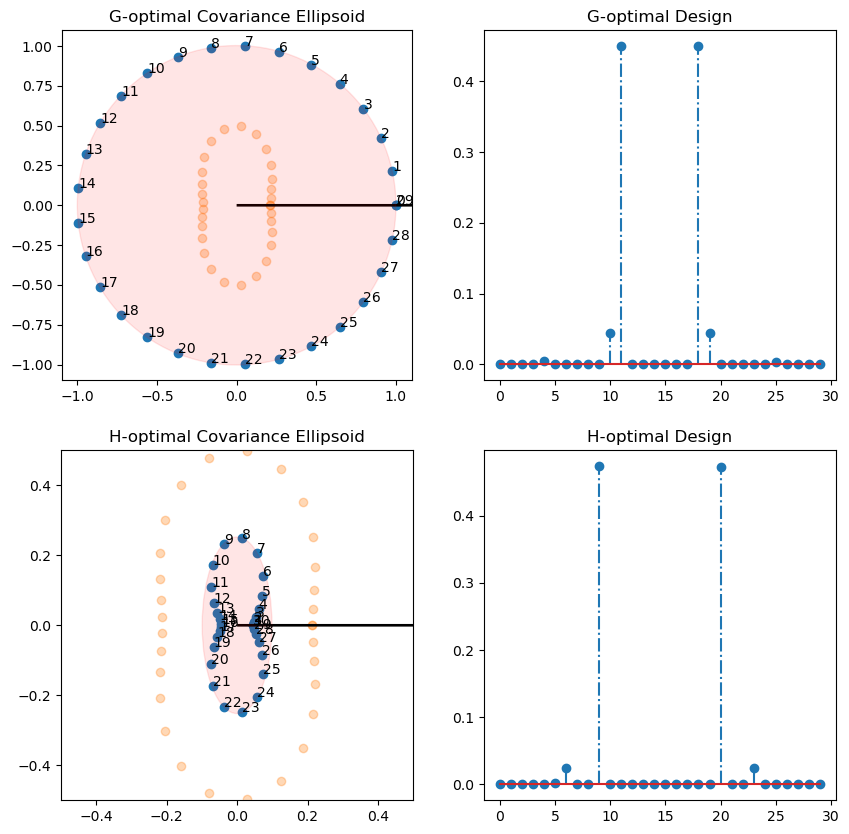}
    \caption{Contrast of G and H optimal designs.}
    \label{fig:exp_designs}
\end{figure}

We consider a setting where $\mc{X} = \{(\cos(2\pi i/30),\sin(2\pi i/30))\}_{i=0}^{29}$ and $\theta^{\ast} = (3,0)$. We compute the optimal designs, $g^{\ast} = \min_{\lambda\in \triangle_{\mc{X}}} \max_{x\in \mc{X}} \|x\|^2_{H(\lambda)^{-1}}$ with corresponding design $\lambda_G\in \triangle_{\mc{X}}$ and $h^{\ast} = \min_{\lambda\in \triangle_{\mc{X}}} \max_{x\in \mc{X}} \dot\mu(x^{\top}\theta)^2\|x\|^2_{H(\lambda)^{-1}}$ with corresponding design $\lambda_{H}\in \triangle_{\mc{X}}$. In the left hand plots, in the top the points in blue correspond to $\mc{X}$, where as the points in the bottom plot correspond to $\{\dot\mu(x^{\top}\theta)x_i\}_{i=0}^{29}$. The orange points in both plots correspond to $\{\sqrt{\dot\mu(x^{\top}\theta)}x_i\}_{i=0}^{29}$. Finally, the red ellipses correspond to $\{x\in \mathbb{R}^2:x^{\top} H(\lambda_G)^{-1}x\leq g^{\ast} \}$ in the top plot, and $\{x\in \mathbb{R}^2:x^{\top} H(\lambda_H)^{-1}x\leq h^{\ast} \}$. 

The right column of plots are stem plots of the resulting design $\lambda_G$ and $\lambda_H$. Note that they are both fairly sparse allocations.

We can see the impact of scaling on the resulting design, namely in the $G$-optimal design we place most of our mass on points $10, 11$ and $18, 19$. In the $H$-optimal design we put our mass on $6, 9$ and $20, 23$ since points that are collinear to $\theta$ get scaled closer to 0 and so our mass is supported along points closer to the vertical axis. In particular, the ellipse stemming from the $H$-optimal design is far smaller than the ellipse stemming from the $G$-optimal design, demonstrating that the value of $h^\ast \ll g^\ast$. 





\end{document}